\documentclass[11pt]{article}

\usepackage{amsmath,amsfonts,amsthm,amssymb,bbm}
\usepackage{algorithm}	    %
\usepackage{algorithmic}	    %
\usepackage{graphicx,subcaption}
\usepackage{textcomp}

\usepackage{epigraph}
\usepackage{soul}
\usepackage{upgreek}

\usepackage[utf8]{inputenc} %
\usepackage[T1]{fontenc}    %
\usepackage{hyperref}       %
\usepackage{url}            %
\usepackage{booktabs}       %
\usepackage{multirow}	    %
\usepackage{nicefrac}       %
\usepackage{microtype}      %
\usepackage{xcolor}         %
\usepackage{bookmark}

\usepackage{pgfplots}
\usepackage{amsfonts}
\usepackage{pgfplotstable}
\usetikzlibrary {arrows.meta}

\pgfplotsset{compat=1.17}

\usepackage{tabularx}
\usepackage{array, collcell}
\usepackage{makecell}
\usepackage{amsmath}
\setcellgapes{3pt}\makegapedcells

\newcommand\AddLabel[1]{%
  \refstepcounter{equation}%
  (\theequation)%
  \label{#1}%
}
\newcolumntype{M}{>{\hfil$\displaystyle}X<{$\hfil}} %
\newcolumntype{L}{>{\collectcell\AddLabel}r<{\endcollectcell}}

\newtheorem*{rep@theorem}{\rep@title}
\newcommand{\newreptheorem}[2]{%
\newenvironment{rep#1}[1]{%
 \def\rep@title{#2 \ref{##1}}%
 \begin{rep@theorem}}%
 {\end{rep@theorem}}}
\makeatother

\newtheorem*{rep@lemma}{\rep@title}
\newcommand{\newreplemma}[2]{%
\newenvironment{rep#1}[1]{%
 \def\rep@title{#2 \ref{##1}}%
 \begin{rep@lemma}}%
 {\end{rep@lemma}}}
\makeatother

\newtheorem*{rep@corollary}{\rep@title}
\newcommand{\newrepcorollary}[2]{%
\newenvironment{rep#1}[1]{%
 \def\rep@title{#2 \ref{##1}}%
 \begin{rep@corollary}}%
 {\end{rep@corollary}}}
\makeatother

\theoremstyle{plain}
\newtheorem{theorem}{{\bf Theorem}}[section]
\newreptheorem{theorem}{{\bf Theorem}}
\newtheorem{lemma}[theorem]{{\bf Lemma}}
\newreptheorem{lemma}{{\bf Lemma}}

\newreptheorem{proposition}{{\bf Proposition}}
\newtheorem{corollary}[theorem]{{\bf Corollary}}
\newreptheorem{corollary}{{\bf Corollary}}
\newtheorem{definition}{{\bf Definition}}[section]

\newtheorem{remark}[theorem]{{\bf Remark}}

\usepackage{authblk}
\usepackage[title]{appendix}
\usepackage[margin=1in]{geometry}

\usepackage[square,numbers]{natbib}
\usepackage{cleveref}

\providecommand{\Comments}{0}  %
\ifnum\Comments=1
\usepackage[colorinlistoftodos,prependcaption,textsize=scriptsize]{todonotes}
\paperwidth=\dimexpr\paperwidth + 2cm\relax
\marginparwidth=\dimexpr\marginparwidth + 1.7cm\relax
\else
\usepackage[disable]{todonotes}
\fi

\newcommand{\mytodo}[1]{\ifnum\Comments=1{#1}\fi}

\newcommand{\tableoftodos}{\ifnum\Comments=1 \listoftodos[Comments/To Do's] \fi}

\definecolor{Gred}{RGB}{219, 50, 54}
\definecolor{Ggreen}{RGB}{60, 186, 84}
\definecolor{Gblue}{RGB}{72, 133, 237}
\definecolor{Gyellow}{RGB}{247, 178, 16}
\definecolor{ToCgreen}{RGB}{0, 128, 0}
\definecolor{myGold}{RGB}{231,141,20}
\definecolor{myBlue}{rgb}{0.19,0.41,.65}
\definecolor{myPurple}{RGB}{175,0,124}

\title{Laplace Transform Interpretation of Differential Privacy}
\author{%
    \textbf{Rishav Chourasia}$^{1,2}$\thanks{Corresponding author}
    \quad
    \textbf{Uzair Javaid}$^{1}$
    \quad
    \textbf{Biplap Sikdar}$^{2}$
    \\
    $^1$Betterdata \quad $^2$National University of Singapore \\
    \texttt{\{rishav, uzair\}@betterdata.ai} \quad
    \texttt{bsikdar@nus.edu.sg}
}
\date{}

\newcommand{\eps}{\varepsilon}
\newcommand{\del}{\delta}
\newcommand{\q}{q}
\newcommand{\n}{n} %
\newcommand{\rhO}{\rho}

\newcommand{\M}{\mathcal{M}} %
\newcommand{\D}{D} %
\newcommand{\UniLap}[1]{\mathcal{L}\left\{#1\right\}}
\newcommand{\BiLap}[1]{\mathcal{B}\left\{#1\right\}}

\newcommand{\X}{\mathcal{X}}
\newcommand{\OO}{\Omega}

\newcommand{\R}{\mathbb{R}}
\newcommand{\Com}{\mathbb{C}}

\newcommand{\Id}{\mathit{I}_d}

\newcommand{\Thet}{\Theta} %
\newcommand{\thet}{\theta} %
\newcommand{\priv}[2]{L_{#1\vert#2}}
\newcommand{\privloss}[2]{\mathrm{PLD}(#1\Vert#2)}
\newcommand{\Z}{Z} %
\newcommand{\Zc}{{Z'}} %

\newcommand{\dirac}{\triangle}
\newcommand{\convolve}{\circledast}

\newcommand{\prob}[2]{\underset{#1}{\mathbb{P}}\left[#2\right]}
\newcommand{\expec}[2]{\underset{#1}{\mathbb{E}}\left[#2\right]}
\newcommand{\indic}[1]{\mathbb{I}\{#1\}}

\newcommand{\Ren}[3]{\mathrm{R}_{#1}\left(#2\middle\Vert#3\right)}	%
\newcommand{\Eren}[3]{\mathrm{E}_{#1}\left(#2\middle\Vert#3\right)}	%

\newcommand{\eqdef}{:=}

\newcommand{\dif}[1]{\mathrm{d} #1}

\begin{document}

\maketitle

\begin{abstract}

    We introduce a set of useful expressions of Differential Privacy (DP) notions in terms of the Laplace transform of the privacy loss distribution. Its bare form expression appears in several related works on analyzing DP, either as an integral or an expectation. We show that recognizing the expression as a Laplace transform unlocks a new way to reason about DP properties by exploiting the duality between time and frequency domains. Leveraging our interpretation, we connect the $(\q, \rhO(\q))$-Rényi DP curve and the $(\eps, \del(\eps))$-DP curve as being the Laplace and inverse-Laplace transforms of one another. This connection shows that the R\'enyi divergence is well-defined for complex orders $\q = \gamma + i \omega$. Using our Laplace transform-based analysis, we also prove an adaptive composition theorem for $(\eps, \del)$-DP guarantees that is exactly tight (i.e., matches even in constants) for all values of $\epsilon$. Additionally, we resolve an issue regarding symmetry of $f$-DP on subsampling that prevented equivalence across all functional DP notions.

    \end{abstract}

\section{Introduction}
\label{sec:intro}
Differential privacy (DP)~\citep{dwork2006differential} has become a widely adopted standard for quantifying privacy of algorithms that process statistical data. In simple terms, differential privacy bounds the influence a single data-point may have on the outcome probabilities. Being a statistical property, the design of differentially private algorithms involves a \emph{pen-and-paper analysis} of any randomness internal to the processing that obscures the influence a data-point might have on its output. A clear understanding of the nature of differential privacy notions is therefore tantamount to study and design of privacy-preserving algorithms.

Throughout its exploration, various functional interpretations of the concept of differential privacy have emerged over the years. These include the privacy-profile curve $\delta(\epsilon)$~\citep{balle2020privacy}  that traces the $(\epsilon, \delta)$-DP point guarantees, the $f$-DP~\citep{dong2019gaussian} view of worst-case \emph{trade-off curve} between type I and type II errors for hypothesis testing membership~\citep{kairouz2015composition,balle2020hypothesis}, the Rényi DP~\citep{mironov2017renyi} function of order $\q$ that admits a natural analytical composition~\citep{abadi2016deep,mironov2017renyi}, the view of the \emph{privacy loss distribution (PLD)}~\citep{sommer2018privacy} that allows for approximate numerical composition~\citep{koskela2020computing,gopi2021numerical}, and the recent \emph{characteristic function formulation} of the dominating privacy loss random variables~\citet{zhu2022optimal}. Each of these formalisms have their own properties and use-cases, and none of them seem to be superior in all aspects. 

Regardless of their differences, they all have some shared difficulties---certain types of manipulations on them are harder to perform in the time-domain, but considerably simpler to do in the frequency-domain. For instance, \citet{koskela2020computing} noted that composing PLDs of two mechanisms involve convolving their probability densities, which can be numerically approximated efficiently by multiplying their Discrete Fast-Fourier Transformations (DFFT) and then inverting it back to get the convolved density using Inverse-DFFT. Such maneuvers are also frequently performed for analytical reasons while proving properties of differential privacy, often without even realizing this detour through the frequency-domain. A notable example of this is the analysis of Moments' accountant by~\citet{abadi2016deep}, where the authors bound higher-order moments of subsampled Gaussian distributions, compose the moments through multiplication, and then derive the $(\eps,\del)$-DP bound on the DP-SGD mechanism. Their analysis goes the through frequency space, as the moment generating function of a random variable corresponds to the two-sided Laplace transform of its probability density function~\citep{miller1951moment}. Often times when dealing with a functional notion of DP, expressing components like expectations or cumulative densities their integral form ends up being a Fourier or a Laplace transform. Realizing them as such can be tremendously useful in analysis.

In this paper, we formalize these time-frequency domain dualisms enjoyed by the functional representations into a new interpretation of differential privacy. In addition to augmenting existing perspectives on DP, this interpretation provides a flexible analytical toolkit that greatly extends our cognitive reach in reasoning about DP and its underpinnings. This interpretation is based on recognizing that the privacy-profile $\del_{P|Q}(\eps) := \sup_S P(S) - e^\eps \cdot Q(S)$ and the R\'enyi-divergence $\Ren{\q}{P}{Q} := \frac{1}{\q-1} \int_\OO P^\q Q^{1-\q} \dif{\thet}$ between any two distributions $P, Q$ on the same space $\OO$ can be seen as a Laplace transform\footnote{Laplace transform maps a time-domain function $g(t)$ with $t \in \R$ to a function $\UniLap{g}(s) := \int_0^\infty e^{-s t} g(t) \dif{t}$ with $s \in \Com$ in the complex space. Similarly, bilateral Laplace transform of $g(t)$ is defined as $\BiLap{g}(s) := \int_{-\infty}^\infty e^{-s t} g(t) \dif{t}$.} of the privacy loss distribution $\privloss{P}{Q}$, the distribution of privacy loss random variable $\Z = \priv{P}{Q}(\Thet)$ where $\Thet \sim P$:
\begin{align}
	\forall \eps \in \R \ &: \ \del_{P|Q}(\eps) = \expec{\Z \leftarrow \privloss{P}{Q}}{\max\{0, 1 - e^{\eps - \Z} \}}  = \UniLap{1 - F_\Z(t + \eps)}(1), \label{eqn:profile_common_expr} \\
	\forall \q \in \Com \ &: \ e^{(\q -1)\cdot\Ren{\q}{P}{Q}} = \expec{\Z \leftarrow \privloss{P}{Q}}{e^{(\q - 1) \cdot \Z}}  = \BiLap{f_\Z(t)}(1 - \q), \label{eqn:renyi_common_expr}
\end{align}
where $F_\Z(t) = \Pr[\Z < t]$ is the cumulative distribution function and $f_\Z(t) = \int_{\{\thet \in \OO : \priv{P}{Q}(\thet)=z\}} P \dif{\thet}$ is the (generalized) density function of the privacy loss random variable $\Z$. The first equality in~\eqref{eqn:profile_common_expr} is a widely used way to represent the $(\eps, \del(\eps))$-DP curve in literature~\citep{sommer2018privacy,balle2020hypothesis,balle2020privacy,koskela2020computing,gopi2021numerical,steinke2022composition,canonne2020discrete}. Similarly, the first equality in~\eqref{eqn:renyi_common_expr} represents the well-known moment-generating function of privacy loss~\citep{mironov2017renyi,abadi2016deep,balle2020hypothesis}. The second equalities above are part of a set of Laplace expressions presented in this paper. Together, these expressions unlock a formal approach to perform a wide-variety of manipulations and transformations on them using the fundamental properties of the Laplace functional (see~\autoref{tab:laplaceTransformProp}). Using them, we show that the privacy-profile and R\'enyi divergence between any two distributions have the following equivalence.
\begin{equation}
	\forall \q  \in \Com \ : \ e^{(\q-1)\cdot\Ren{\q}{P}{Q}} = \q(\q-1) \cdot \BiLap{\del_{P|Q}(t)}(1-\q),
\end{equation}
which again is a Laplace transform expression. Furthermore,~\citet{zhu2022optimal}'s characteristic function of the privacy loss $\phi_{P|Q}(\q) := \expec{P}{e^{i\q \log(P/Q)}}$ also turns out to be a Fourier transform, which is a special case of the bilateral Laplace transform:
\begin{equation}
	\label{eqn:zhu_char}
	\forall \q  \in \R\ : \ \phi_{P|Q}(\q) = \expec{\Z \leftarrow \privloss{P}{Q}}{e^{i\q \Z}} = \BiLap{f_\Z}(-i\q).
\end{equation}

These expressions can take advantage of the relationship between their time-domain and complex frequency-domain representations, as certain manipulations are more straightforward in one domain as compared to the other. Using the Laplace transform interpretations extensively, our paper presents the following findings.
\begin{enumerate}
	\item We note that the Laplace transform expression of R\'enyi divergence permits the order $\q$ to be a complex number in $\Com$. Based on this observation, we revisit the discussion on equivalence and interconversion between $(\q, \rhO)$-R\'enyi DP and $(\eps, \del)$-DP in literature~\citep{balle2020hypothesis,zhu2022optimal,asoodeh2020better,canonne2020discrete}. We show that the privacy-profile curve $\del_{P|Q}(\eps)$ and the R\'enyi divergence $\Ren{\q}{P}{Q}$ as a function of $\q$ are equivalent as long as either $P$ is absolutely continuous w.r.t. $Q$ (denote as $P \ll Q$) or $Q \ll P$; absolute continuity in both directions is not necessary. Moreover, we establish that while $\del_{P_1|Q_1}(\eps) \leq \del_{P_2|Q_2}(\eps)$ for all $\eps$ implies that $\Ren{\q}{P_1}{Q_1} \leq \Ren{\q}{P_2}{Q_2}$ for all $\q > 1$, the converse does not hold. This is due to the fact that the dominance relationship between privacy profiles $\del_{P_1|Q_1}(\eps)$ and $\del_{P_2|Q_2}(\eps)$ depends on how the R\'enyi divergence curves $\Ren{\q}{P_1}{Q_1}$ and $\Ren{\q}{P_2}{Q_2}$ behave along the complex line $\{\q \in \Com : \mathfrak{Re}(\q) = c\}$ at any $c\in \R$ for which the two divergences exist; not along $(1, \infty)$.

	\item Among all functional notions of DP, \emph{exactly tight adaptive} composition theorem is only known for R\'enyi DP in an explicit form\footnote{\citet{dong2019gaussian} examine tight composition under the $f$-DP framework by defining an abstract composition operation, denoted $f_1 \otimes f_2$. However, they do not provide an explicit form for this operator for a general trade-off function $f$, offering it only for the specific case of the Gaussian trade-off function $G_\mu$.}\citep{mironov2017renyi,bun2016concentrated}. And, for the PLD formalism, only non-adaptive composition theorems are known that are \emph{exactly tight}\footnote{Unlike under non-adaptivity, composing two privacy loss random variables $ Z_1 $ and $ Z_2 $ does not amount to convolving their privacy loss distributions (PLDs) \( f_{Z_1} \circledast f_{Z_2} \) when the mechanisms are adaptive because $Z_1, Z_2$ become dependent. We note that \citet{gopi2021numerical} seem to incorrectly assert their Theorem 5.5 to be valid under adaptivity.}~\citep{gopi2021numerical,sommer2018privacy,koskela2020computing}.
	In this paper, we establish an exactly tight theorem for composing any two privacy profiles, $\del_{P_1|Q_1}(\eps)$ and $\del_{P_2|Q_2}(\eps)$, leveraging time-frequency dualities with their R\'enyi divergence curves. Our composition method also extends to adaptive scenarios, provided that the conditional distributions $P_2^\thet$ and $Q_2^\thet$, given an observation $\thet$ from the first distribution pair, are dominated by a privacy profile for all $\thet$, which is a standard assumption for adaptive composition guarantees.

	\item We apply our composition theorem for privacy profiles to derive an optimal composition theorem for \((\epsilon_i, \delta_i)\)-point guarantees of differential privacy. Our approach begins by determining the worst-case privacy profile \(\delta_i(\epsilon)\) that \emph{any} \((\epsilon_i, \delta_i)\)-DP mechanism must satisfy. We then use our composition theorem to derive the combined privacy profile \(\delta^\otimes(\epsilon)\). This provides the most precise composition guarantee possible when given only that a sequence of mechanisms each satisfies an \((\epsilon_i, \delta_i)\)-DP point guarantee. Our bound surpasses the optimal composition result in \citet[Theorem 3.3]{kairouz2015composition} because, whereas their result only provides a discrete set of \((\epsilon, \delta)\) values met by the composed curve, ours forms a continuous curve. This continuity enables us to determine the tightest \(\epsilon\) value for any given \(\delta\) budget. We also show that our results align with the bounds generated by numerical accountants such as Google’s PLDAccountant \citep{doroshenko2022connect} and Microsoft’s PRVAccountant \citep{gopi2021numerical}.

	\item The concept of \( f \)-DP introduced by \citet{dong2019gaussian} provides a functional perspective on the indistinguishability between two distributions \( P \) and \( Q \) through hypothesis testing. The function \( f: [0, 1] \rightarrow [0, 1] \) represents a bound on the trade-off \( T(P, Q) : [0, 1] \rightarrow [0, 1] \) between Type-I and Type-II errors for any test aimed at determining whether a sample \( \thet \) originates from \( P \) or \( Q \). Unlike other functional notions of differential privacy, \( f \)-DP is unique in being \emph{not connected} to the rest via a Laplace transform. Instead, \citet{dong2019gaussian} establish that the privacy profile \( \del(\eps) \) and the trade-off curve \( f \) of a mechanism exhibit a convex-conjugate relationship, also known as Fenchel duality. However, \citet[Proposition 2.12]{dong2019gaussian} confirm this functional equivalence only when \( f \) is \emph{symmetric}. With Poisson subsampling at probability \( p \), the resulting amplified curve \( f_p(x) = p f(x) + (1 - p) \cdot x \) becomes asymmetric. To ensure symmetry, the subsampling result \citep[Theorem 4.2]{dong2019gaussian} applies a \( p \)-sampling operator \( C_p(f) : \min\{f_p, f_p^{-1}\}^{**} \) that overestimates the \( f_p \)-curve. We show that this symmetrization step disrupts the equivalence between \( f \)-DP and privacy profile formalisms (and thereby with other functional notions). To address this, we propose maintaining the natural asymmetry in \( f \)-DP and avoid the need for this symmetrization step by adopting a convention on the direction of skew. This completes the equivalences across all functional notions of DP.

\end{enumerate}

\textbf{Related work:} Our work builds an interpretation of differential privacy by leveraging several works that appeared before. This includes, but not limited to the works on various interpretations of privacy by~\citet{dong2019gaussian, dwork2016concentrated, dwork2014algorithmic, mironov2017renyi, bun2016concentrated, sommer2018privacy, gopi2021numerical, koskela2020computing, zhu2022optimal}. In particular, the work of~\citet{zhu2022optimal} shares the most similarity with ours, as they were the first to observe that many functional notions of differential privacy appear to be linked via Laplace or Fourier transforms. However, their work centers on using the characteristic function of privacy loss (in \eqref{eqn:zhu_char}) as an intermediate functional representation connecting various DP notions. In contrast, we examine the nature of these connections themselves to harness the perspective of Laplace transformations as an analytical toolkit for differential privacy.

Relevant studies on composition theorems for differential privacy include~\citet{dwork2010boosting, kairouz2015composition, murtagh2015complexity, bun2016concentrated, mironov2017renyi}, along with numerical accounting methods such as those by~\citet{gopi2021numerical, koskela2020computing, doroshenko2022connect}.

\textbf{Paper structure:} After reviewing preliminaries on DP and Laplace transforms in \autoref{sec:prelim}, we present an equivalent description of both the $\del_{P|Q}(\eps)$ privacy profile and the $\Ren{\q}{P}{Q}$-R\'enyi divergence curve in terms of a set of Laplace transforms of the privacy loss distribution's probability function in~\autoref{sec:characterization} and use them to connect the two notions. After discussing the implications of these connections, we provide our composition results for privacy profiles and $(\eps, \del)$-DP point guarantees in~\autoref{sec:composition}. Finally, in~\autoref{sec:subsampling} we discuss the problem of asymmetry in functional notions and an approach to handling it without breaking equivalences.

\section{Preliminaries}
\label{sec:prelim}
Here we introduce our notations, provide some background related to Differential Privacy and a short introduction to Laplace transforms.

\subsection{Background on Differential Privacy}
\label{ssec:dp}

For a data universe \(\X\), we consider datasets \(\D\) of size \(n\): \(\D = (x_0, \ldots, x_n) \in \X^n\) and algorithms \(\M: \X^n \rightarrow \OO\) that return a random output in space \(\OO\). We assume that \(\X\) includes a sentinel element \(\bot\), representing an empty entry, to simulate `add' and `remove' adjacency within `replace' adjacency. Two datasets \(\D\) and \(\D'\) in \(\X^n\) are considered \emph{neighboring} (denoted by \(\D \simeq \D'\)) if they differ by a single record replacement. Throughout the paper, we denote the distributions of the output random variables \(\M(\D)\) and \(\M(\D')\) as \(P\) and \(Q\), respectively, and focus our study on the indistinguishability behavior of these two distributions. For simplicity, we use the same symbols \(P\) and \(Q\) to refer to their probability mass or density functions.

\begin{definition}[$(\eps, \del)$-Differential Privacy and Privacy Profiles~\citep{dwork2014algorithmic, balle2020privacy}]
	\label{def:dp}
	Let $\eps \geq 0$ and $0\leq \del \leq 1$. A randomized algorithm ${\M}$ is $(\eps, \del)$-\emph{differentially private} (hereon $(\eps,\del)$-DP) 
	\begin{equation}
		\label{eqn:dp}
		\text{if}\quad \forall \D \simeq \D' \quad \text{and}\quad \forall S \subset \OO, \quad \mathbb{P}[\M(\D) \in S] \leq e^\eps \cdot \mathbb{P}[\M(\D') \in S] + \del.
	\end{equation}
	The \emph{privacy profile} of $\M$ on a pair of neighbours $\D\simeq \D'$ is the function $\del_{P|Q}(\eps)$ for $\eps \in \R$, where
	\begin{equation}
		\label{eqn:profile}
		\del_{P\vert Q}(\eps) \eqdef \sup_{S \subset \OO} P(S) - e^\eps Q(S) = \int_\OO \max\{0, P(\thet) - e^\eps \cdot Q(\thet)\}d\thet.
	\end{equation}
	For any $\eps > 0$, algorithm $\M$ \emph{tightly} satisfies $(\eps, \del)$-DP if $\del = \sup_{\D \simeq \D'} \del_{\M(\D)|\M(\D')}(\eps)$.
\end{definition}
\begin{remark}
	\label{rem:dp_reversal}
	The privacy profile \(\del_{P|Q}(\eps)\) is equivalent to the Hockey-stick divergence \(H_{e^\epsilon}(P \Vert Q)\). Our definition of the privacy profile allows \(\eps < 0\), in contrast to the original definitions of privacy profile by \citet{balle2020privacy} and of Hockey-stick divergence by \citet{sason2016f}, which consider only non-negative \(\eps\). Although permitting \(\eps < 0\) might seem counterintuitive, it is both accurate and efficient, as negative values of \(\eps\) yield the privacy profile with \(P\) and \(Q\) reversed.
	\begin{equation}
		\label{eqn:remark_profile}
 		\del_{Q\vert P}(\eps) = \sup_{S \subset \OO} Q(S) - e^\eps P(S) = 1 - e^{\eps} + e^{\eps}[\sup_{S \subset \OO} P(S^\complement) - e^{-\eps} Q(S^\complement)]  = 1 - e^{\eps} + e^{\eps} \del_{P\vert Q}(-\eps).
 	\end{equation}
\end{remark}

\begin{definition}[R\'enyi Differential Privacy~\citep{mironov2017renyi}]
	\label{def:rdp}
	Let $\q > 1$ and $\rhO \geq 0$. A randomized algorithm $\M$ is $(\q, \rhO)$-\emph{R\'enyi differentially private} (henceforth $(\q, \rhO)$-R\'enyi DP) if, for all datasets $\D \simeq \D'$, the $\q$-R\'enyi divergence satisfies $\Ren{\q}{\M(\D)}{\M(\D')} \leq \rhO$.
	For two distributions $P, Q$ over the same space, we define the R\'enyi divergence $\Ren{\q}{P}{Q}$ of any order $\q \in \mathrm{ROC}_{P, Q}$ as
	\begin{align}
		\label{eqn:rdp}
		\Ren{\q}{P}{Q} \eqdef \frac{1}{\q-1} \log \Eren{\q}{P}{Q}, \quad \text{ where }\quad
		\Eren{\q}{P}{Q} \eqdef \int_{\thet \in \OO} P(\thet)^\q Q(\thet)^{1-\q} d\thet,
	\end{align}
	and $\mathrm{ROC}_{P, Q}$ is the region consisting of all orders $\q \in \Com$ where the integral is \emph{conditionally convergent}.

\end{definition}
\begin{remark}
	\label{rem:rdp_reversal}
	It is known that R\'enyi divergence converges to Kullback-Leibler divergence as order $\q$ tends to one\footnote{It follows from the \emph{replica trick}: $\expec{}{\log X} = \lim_{n \rightarrow 0} \frac{1}{n} \log \expec{}{X^n}$, when $\log X$ is the privacy loss rv $\Z \leftarrow \privloss{P}{Q}$.}. Among real orders $\q \in \R$, privacy researchers typically restrict themselves to $\q > 1$ without thinking much about values smaller than $1$. Just like privacy profile $\del_{P|Q}(\eps)$, R\'enyi divergence for orders $\q < 1$ yield the R\'enyi divergence with $P$ and $Q$ reversed.
	\begin{align}
		e^{(\q - 1) \cdot \Ren{\q}{P}{Q}} = \Eren{\q}{P}{Q} = \int_{\thet \in \OO} P(\thet)^\q Q(\thet)^{1-\q} d\thet = \Eren{1-\q}{Q}{P} = e^{-\q \cdot \Ren{1-\q}{Q}{P}}.
	\end{align}
\end{remark}
\begin{definition}[$f$-Differential Privacy~\citep{dong2019gaussian}]
	\label{dfn:fdp}
	Let $f: [0, 1] \rightarrow [0, 1]$ be a convex, continuous, non-increasing function such that $f(x) \leq 1 - x$ for all $x \in [0, 1]$. A randomized algorithm $\M$ is \emph{$f$-differentially private} (henceforth $f$-DP) if 
	\begin{equation}
		\label{eqn:fdp_bound}
		\forall \alpha \in [0, 1] : f_{\M(\D)|\M(\D')}(\alpha) \geq f(\alpha),
	\end{equation}
	where for any distributions $P, Q$ on $\OO$, the $f_{P|Q}: [0, 1] \rightarrow [0, 1]$ is the trade-off function defined as
	\begin{equation}
		\label{eqn:fdp}
		\forall \alpha \in [0, 1] : f_{P|Q}(\alpha) \eqdef \inf_{\phi: \OO \rightarrow [0, 1] } \{\beta_\phi : \alpha_\phi \leq \alpha \},
	\end{equation}
	where $\alpha_\phi \eqdef \expec{P}{\phi}$, and $\beta_\phi \eqdef 1 - \expec{Q}{\phi}$.
\end{definition}
\begin{remark}
	\label{rem:f_reversal}
	For a hypothesis test \(\phi\) on a sample \(\thet\) originating from either \(P\) or \(Q\), we follow the convention that \(\thet \sim Q\) represents the positive event, and \(\phi(\thet) = 1\) denotes a positive prediction. With this convention, \(\alpha_\phi\) and \(\beta_\phi\) represent the \emph{false positive rate} (Type I error) and \emph{false negative rate} (Type II error), respectively. Note that the left-continuous inverse of the trade-off function $f_{P|Q}^{-1}(\beta)$ gives the trade-off curve with $P$ and $Q$ reversed.
	\begin{equation}
		f_{P|Q}^{-1}(\beta) := \{\inf \alpha \in [0, 1] : f_{P|Q}(\alpha) \leq \beta \} = \inf_{\phi: \OO \rightarrow [0, 1]} \{\alpha_\phi : \beta_\phi \leq \beta \} = f_{Q|P}(\beta).
	\end{equation}
\end{remark}

\begin{theorem}[Privacy of Gaussian Mechanism~\citep{dwork2014algorithmic,balle2018improving,mironov2017renyi}]
	\label{thm:gaussian_privacy}
	If $P = \mathcal{N}(\mu, \sigma^2 \Id)$ and $Q = \mathcal{N}(\mu', \sigma^2 \Id)$ are two multivariate Gaussian distributions on $\R^d$ such that $\kappa = \Vert \mu - \mu'\Vert_2^2/2\sigma^2$, then $\Ren{\q}{P}{Q} = \kappa \q$ for all $\q > 1$ and for all $\eps \in \R$~\footnote{\footnotesize While the original result by~\citet{balle2018improving} was stated only for non-negative $\eps$, the expression remains identical for $\eps < 0$ on invoking~\eqref{eqn:remark_profile}, thanks to the symmetry property $\overline\Phi(-x) = 1 - \overline\Phi(x)$ of normal distribution.}, 
	\begin{equation}
		\label{eqn:gaussian_privacy_profile}
		\del_{P|Q}(\eps) = \overline{\Phi}\left(\frac{\eps - \kappa}{\sqrt{2\kappa}}\right) - e^\eps \overline{\Phi}\left(\frac{\eps + \kappa}{\sqrt{2\kappa}}\right) = O(e^{-\eps^2/4\kappa}), \text{   where   } \overline\Phi(t) = \prob{G \sim \mathcal{N}(0,1)}{G > t}.
	\end{equation}
\end{theorem}

\subsection{Background on Laplace Transforms}
\label{ssec:laplace_transform}

Laplace transform is an integral transform that maps time domain functions with real arguments $(t \in \R)$ to frequency domain functions with complex arguments ($s \in \Com$). The one-sided and two-sided Laplace transformations of a function $g(t)$ at complex frequency $s$ is defined respectively as
\begin{equation}
	\label{eqn:app_laplace_def}
	\UniLap{g(t)}(s) \eqdef \int_{0^+}^\infty e^{-st} g(t) dt,\quad \text{and}\quad \BiLap{g(t)}(s) \eqdef \int_{-\infty}^\infty e^{-st} g(t) dt,
\end{equation}
where $0^+$ is the shorthand notation for limit approaching $0$ from positive side. We can express two-sided Laplace transforms using one-sided transform as 
\begin{equation}
	\label{eqn:bi_to_unilap}
	\BiLap{f(t)}(s) = \UniLap{f(t)}(s)  + \int_{0^-}^{0^+} f(t)dt + \UniLap{f(-t)}(-s),
\end{equation}
where $\int_{0^-}^{0^+} f(t)dt$ may not be $0$ if $f$ has an \emph{impulse} (aka.  \emph{integrable singularity}) at $0$, informally defined to be an infinitely dense point but with a finite mass. 
\begin{remark}
	Conventionally, one-sided Laplace transform is defined to include point mass located at $0$ entirely (i.e., integration is from $0^-$ instead of $0^+$). For aligning with the conventions in differential privacy, our definition entirely excludes the point mass located at $0$ by convention. The Laplace transform properties presented in this paper are adjusted to reflect the same.
\end{remark}

The values of $s \in \Com$ for which the integrals in~\eqref{eqn:app_laplace_def} converges conditionally\footnote{Conditional convergence for $\UniLap{g}$ at $s \in \Com$ means that the limit \(\lim_{\gamma \rightarrow \infty} \int_{0^+}^{\gamma} e^{-st} g(t) \, \mathrm{d}t\) exists. Similarly,  the limit \(\lim_{\gamma \rightarrow \infty} \int_{-\gamma}^{\gamma} e^{-st} g(t) \, \mathrm{d}t\) should exist for $\BiLap{g}$ to be conditionally convergent at $s$.} is referred to as the \emph{region of (conditional) convergence} (ROC) of the respective transforms. This region is always a strip parallel to the imaginary axis $\omega i$ where $i = \sqrt{-1}$ and $\omega \in \R$, which follows from dominated convergence theorem~\citep{oppenhiem1996signals}. We denote the region of convergence for a function $g$ as $\mathrm{ROC}_{\mathcal{L}}\{g\}$ in case of one-sided Laplace transform and as $\mathrm{ROC}_{\mathcal{B}}\{g\}$ for two-sided Laplace transforms.

\paragraph{Uniqueness and Inversion.} Laplace transforms are unique \citep{cohen2007numerical}: if two continuous functions \(g\) and \(h\) share the same Laplace transform, then, \(g(t) = h(t)\) holds for all \(t \in \mathbb{R}^+\) in the case of a one-sided Laplace transform and for all \(t \in \mathbb{R}\) for a two-sided Laplace transform.
As a consequence of uniqueness, Laplace transform $\bar{g}(s) = \UniLap{g(t)}(s)$ can be inverted to get back $g(t)$ by applying an \emph{Inverse Laplace transform}~\citep{orloff2015uniqueness}, defined as
\begin{equation}
	\label{eqn:inverse_laplace}
	g(t) = \mathcal{L}^{-1}\{\bar{g}(s)\}(t) \eqdef \frac{1}{2\pi i} \lim_{\omega \rightarrow \infty} \int_{\gamma - i\omega}^{\gamma + i\omega} e^{st} \bar{g}(s)ds = \frac{1}{2\pi i} \int_{\gamma - i\infty}^{\gamma + i\infty} e^{st} \bar{g}(s)ds,
\end{equation}
where the integral is taken over the line consisting of all points $s$ with $\mathfrak{Re}(s) = \gamma$ for any $\gamma$ lying in the ROC of $\bar{g}(s)$. The formula~\eqref{eqn:inverse_laplace} also inverts two-sided Laplace transform $\BiLap{g(t)}$, as long as we choose $\gamma$ in the ROC of the two-sided transform~\citep{oppenhiem1996signals}.
\begin{remark}
	Uniqueness of Laplace transforms extends to discontinuous functions as well. If $g$ and $h$ are continuous \emph{almost everywhere}, i.e. the set where either isn't continuous has a total Lebesgue measure of zero, then $g(t) = h(t)$ almost everywhere in the respective time-domains. In such cases, it can be shown~\citep{dyke2001introduction} that the inversion formula~\eqref{eqn:inverse_laplace} gives
	\begin{equation}
		\frac{1}{2\pi i} \int_{\gamma - i\infty}^{\gamma + i\infty} e^{st} \bar{g}(s)ds = \frac{1}{2}[g(t^-) + g(t^+)].
	\end{equation}
\end{remark}

\paragraph{Relation to Fourier transform.} The Fourier transform $G(\omega)$ of a function $g$ is defined as
\begin{equation}
	G(\omega) = \int_{-\infty}^\infty e^{-i 2\pi\omega t } g(t)dt,
\end{equation}
which is the same as the two-sided laplace transform $\BiLap{g(t)}(s)$ for a purely imaginary $s = i 2\pi \omega$. As such, Fourier transform is seen as a \emph{special case} of Laplace transforms.

\paragraph{Properties of Laplace transforms.} The Laplace transformation is a very useful tool because a lot of operations in the time-domain correspond to simpler operations in the frequency domain and vice versa. For a detailed exposition on these properties, refer to~\citet{cohen2007numerical} for one-sided Laplace transform and to~\citet{oppenhiem1996signals} for the two-sided counterpart. In Appendix~\ref{app:prop_table}, we provide a table summarizing all the properties that we rely on in this paper. We reference properties of~\autoref{tab:laplaceTransformProp} throughout the paper using the notation $\overset{(m)}{=}$, where $(m)$ is the equation number of the used property.

\section{Laplace Transform Expressions of Differential Privacy}
\label{sec:characterization}

Differential privacy bounds the maximum divergence in the output distribution caused by including or omitting a data-point from the dataset. This principle is mirrored in the notion of \emph{privacy loss distribution} which expresses how much an algorithm's output reveals about the inclusion of a specific data-point.

\begin{definition}[Privacy Loss Distribution~\citep{sommer2018privacy}]
	\label{def:pld}
	The \emph{privacy loss} of an observation $\thet \in \OO$ from an algorithm $\M$, when comparing datasets $\D \simeq \D'$, is defined as $\priv{P}{Q}(\thet) \eqdef \log(P(\thet)/Q(\thet))$, where $P$ and $Q$ are the probability mass/density functions of $\M(\D)$ and $\M(\D')$, respectively.
	The \emph{privacy loss distribution} $\privloss{P}{Q}$ is the distribution of $\priv{P}{Q}(\Thet)$ when $\Thet \sim P$.
\end{definition}
\begin{remark}
	\label{rem:pld_reversal}
	Similar to other functional privacy notions, the PLD formalism possesses a reversal property~\citep[Definition 3.1]{gopi2021numerical}. Let \( f_\Z \) and \( f_{\Z'} \) represent the generalized density functions of the random variables \( \Z \leftarrow \privloss{P}{Q} \) and \( \Z' \leftarrow \privloss{Q}{P} \), respectively. Then
	\begin{equation}
		\forall z \in \R \ : f_{\Z}(z) = e^{z} \cdot f_{\Z'}(-z).
	\end{equation}
\end{remark}
The $\privloss{P}{Q}$ describes how outputs arising from $\D$ increase an observer's confidence that they did not come from $\D'$. Many prior works make use of the following set of DP expressions in terms of the privacy loss distribution~\citep{sommer2018privacy,balle2020hypothesis,balle2020privacy,koskela2020computing,gopi2021numerical,steinke2022composition,canonne2020discrete}. 
The privacy profile $\del_{P|Q}(\eps)$ is expressed as
\begin{equation}
	\label{eqn:dp_via_Z}
	\del_{P|Q}(\eps) = \expec{\privloss{P}{Q}}{1 - e^{\eps - \Z}}_+ = \expec{\privloss{Q}{P}}{e^{-\Z'} - e^\eps}_+ = \Pr_{\privloss{P}{Q}} [ \Z > \eps] - e^\eps \cdot \Pr_{\privloss{Q}{P}} [\Z' < - \eps],
\end{equation}
and the R\'enyi divergence $\Ren{\q}{P}{Q}$ is expressed as
\begin{equation}
	\label{eqn:dp_via_Z_ren}
	\Ren{\q}{P}{Q} = \frac{1}{\q-1}\log\expec{\privloss{P}{Q}}{e^{(\q-1)\cdot\Z}} = \frac{1}{\q-1}\log\expec{\privloss{Q}{P}}{e^{-\q \cdot \Z'}}.
\end{equation}
In the following theorem, we present a more dynamic version of these relationships by expressing them as a set of \emph{Laplace transforms} of the privacy loss distributions $\privloss{P}{Q}$ and $\privloss{Q}{P}$.
\begin{theorem}
	\label{thm:dp_as_laplace}
	For a random variable $X$, let $F_X(t) \eqdef \Pr[X \leq t]$ denote its \emph{cumulative distribution function} and $f_X(t)$ denote its \emph{generalized probability density function}\footnote{\footnotesize We define density as $f_X(t)dt = \lim_{a\rightarrow 0^+}\int_{t-a}^{t+a} F_X(u) du$ to handle cases where $F_X$ isn't differentiable everywhere, such as when PLD is a discrete distribution. This density is expressible with \emph{Dirac delta} $\dirac(t)$ as $f_X(t) = \begin{cases} \dot{F}_X(t) & \text{if derivative $\dot F_X$ exists at $t$} \\ [F_X(t^+) - F_X(t^-)] \dirac(t) & \text{otherwise} \end{cases}$, and satisfies $F_X(t) = \int_{-\infty}^{t^+}f_X(u)du$.}. 
	Let $P$ and $Q$ be probability distributions and $\Z \sim \privloss{P}{Q}$ and $\Z' \sim \privloss{Q}{P}$ denote their privacy loss random variables.
	If $\Z \sim \privloss{P}{Q}$ and $\Z' \sim \privloss{Q}{P}$, then for all $\eps \in \R$, 
	\begin{align}
		\del_{P|Q}(\eps) &= \UniLap{1 - F_\Z(t + \eps)}(1) \label{eqn:dp_as_laplace_1}\\
				 &= e^\eps \cdot \UniLap{F_{\Z'}(-t-\eps)} (-1) \label{eqn:dp_as_laplace_2}\\
				 &= e^\eps \cdot \UniLap{f_{\Z'}(-t-\eps)}(-1) - \UniLap{f_\Z(t + \eps)}(1) \label{eqn:dp_as_laplace_3} \\
				 &= \UniLap{f_\Z(t + \eps)}(0)  - e^\eps \cdot \UniLap{f_{\Z'}(-t-\eps)}(0). \label{eqn:dp_as_laplace_4}
	\end{align}
	And, for all $\q \in \mathrm{ROC}_\mathcal{B}\{f_{\Z'}\}$ (or equivalently, $1-\q \in \mathrm{ROC}_{\mathcal{B}}\{f_{\Z}\}$),
	\begin{align}
		e^{(\q - 1) \cdot \Ren{\q}{P}{Q}} = \Eren{\q}{P}{Q} = \mathcal{B}\{f_{\Z}(t)\}(1-\q) = \mathcal{B}\{f_{\Z'}(t)\}(\q). \label{eqn:rdp_as_laplace}
	\end{align}
\end{theorem}
Laplace expressions in~\autoref{thm:dp_as_laplace} often arise in their explicit integral forms within several proofs in related works on differential privacy, for instance~\citep[Lemma 9]{canonne2020discrete}, \citep[Proposition 7]{steinke2022composition}, and~\citep[Theorem 5]{balle2018improving}. In their integral forms, they frequently undergo manipulations like integration-by-parts or change-of-variables which can quickly get complicated. Our~\autoref{thm:dp_as_laplace} offers a way to simplify the complexity of such manipulations as one can express the concerned terms in their Laplace expressions and invoke its properties from~\autoref{tab:laplaceTransformProp}, like~\eqref{eq:TimeShifting} and~\eqref{eq:Scaling} for shifting or scaling variable of integration and~\eqref{eq:Differentiation} and~\eqref{eq:BilatIntegration} for integrating-by-parts. Examples in this paper will illustrate that reasoning about privacy this way through its Laplace transform interpretation could be quite effective. 

In the following theorem, we show that the R\'enyi divergence $\Ren{\q}{P}{Q}$ and the privacy profile $\del_{P|Q}(\eps)$ are connected through a Laplace transform as well. We can show this using only the expressions in~\autoref{thm:dp_as_laplace}.

\begin{theorem}[R\'enyi DP from privacy profile]
	\label{thm:dp_profile_to_rdp}
	Let $\q >1$. For any two distributions $P$ and $Q$, 
	\begin{equation}
		\label{eqn:dp_profile_to_rdp}
		e^{(\q-1) \cdot \Ren{\q}{P}{Q}} = \Eren{\q}{P}{Q} = \q(\q-1) \cdot \BiLap{\del_{P|Q}(t)}(1-\q),
	\end{equation}
	for all orders $\q$ such that $1-\q \in \mathrm{ROC}_\mathcal{B}\{\del_{P|Q}\}$.
\end{theorem}
\begin{proof}
	Let $\Z \sim \privloss{P}{Q}$ and $\Z' \sim \privloss{Q}{P}$. From~\eqref{eqn:dp_as_laplace_4}, 
	\begin{align}
		\del_{P|Q}(\eps) &= \UniLap{f_\Z(t+\eps)}(0) - e^\eps \cdot \UniLap{f_{\Z'}(-t-\eps)}(0) \\
				 &= \int_{0^+}^\infty e^0 \cdot f_\Z(t + \eps)\mathrm{d}t - e^\eps \cdot \int_{0^+}^\infty e^0 \cdot f_{\Z'}(-t-\eps)\mathrm{d}t \\
				 &= 1 - F_\Z(\eps) - e^\eps \cdot F_{\Z'}(-\eps). \label{eqn:dp_pld}
	\end{align}
	We apply the linearity, time-shifting, and reversal properties of two-sided Laplace transforms to simplify the Laplace transform of privacy profile as follows
	\begin{align}
		\BiLap{\del_{P|Q}(t)}(1-\q) &= \BiLap{1-F_\Z(t) - e^t \cdot F_{\Z'}(-t)}(1-\q) \\
					    &\overset{\eqref{eq:Linearity}}{=} \BiLap{1-F_\Z(t)}(1-\q) - \BiLap{e^t \cdot F_{\Z'}(-t)}(1-\q) \\
					    &\overset{\eqref{eq:FrequencyShifting}}{=} \BiLap{1-F_\Z(t)}(1-\q) - \BiLap{F_{\Z'}(-t)}(-\q) \\
					    &\overset{\eqref{eq:Reversal}}{=} \BiLap{1-F_\Z(t)}(1-\q) - \BiLap{F_{\Z'}(t)}(\q). 
	\end{align}
	Next, we apply the derivative property of Laplace transforms to get
	\begin{equation}
		\BiLap{1-F_\Z(t)}(1-\q) \overset{\eqref{eq:Differentiation}}{=} \frac{\BiLap{f_\Z(t)}(1-\q)}{\q-1}\quad \text{and}\quad \BiLap{F_{\Z'}(t)}(\q) \overset{\eqref{eq:Differentiation}}{=} \frac{\BiLap{f_{\Z'}}(\q)}{\q}.
	\end{equation}
	Finally, noting from Theorem~\ref{thm:dp_as_laplace} and~\eqref{eqn:rdp} that $\BiLap{f_\Z}(1-\q) = \BiLap{f_\Zc}(\q) = \Eren{\q}{P}{Q}$, we have
	\begin{equation}
		\BiLap{\del_{P|Q}(t)}(1-\q) = \Eren{\q}{P}{Q} \left[\frac{1}{\q-1} - \frac{1}{\q}\right] = \frac{e^{(\q-1)\Ren{\q}{P}{Q}}}{\q(\q-1)}.
	\end{equation}
\end{proof}

The privacy profile $\del_{P|Q}$, for any pair of distributions $P,Q$, is a continuous function\footnote{We can write $\del_{P|Q}(\eps) = \sup_S \psi_S(\eps)$ where $\psi_S(\eps) := P(S) - e^\eps \cdot Q(S)$ is a continuous decreasing function. Therefore, the supremum of $\phi_S$ over all $S$ is also a continuous decreasing function.}. Therefore, by Lerch's theorem (cf. discussion on uniqueness in~\autoref{ssec:laplace_transform}), \emph{any two privacy profiles share the same R\'enyi divergence curve if and only if they are identical}. In other words, a privacy profile $\del_{P|Q}$ is equivalent to its R\'enyi divergence curve $\Ren{\q}{P}{Q}$, provided the Laplace transform $\BiLap{\del_{P|Q}}(1-\q)$ exists at some $\q \in \Com$, i.e., $\mathrm{ROC}_\mathcal{B}\{\del_{P|Q}\} \neq \emptyset$.

Note that \autoref{thm:dp_profile_to_rdp} establishes a connection between R\'enyi divergence and privacy profile that applies to \emph{all distributions $P$ and $Q$}, whether $P$ is absolutely continuous\footnote{$P$ is absolutely continuous with respect to $Q$ if for all measurable subsets $S \subset \OO$, $P(S) > 0 \implies Q(S) > 0$. We say that distribution pair $(P,Q)$ is absolutely continuous (w.r.t. each other) if $P \ll Q$ and $Q \ll P$.} (denoted as $P \ll Q$) with respect to $Q$ or not. The choice of $P, Q$ however influences the region of convergence $\mathrm{ROC}_\mathcal{B}\{\del_{P|Q}\}$ where the R\'enyi divergence can be defined. One can verify that if $P \ll Q$, the Laplace transform $\BiLap{\del_{P|Q}}(1-\q)$ converges for all real $\q > 1$; and if $Q \ll P$, it converges for all $\q < 1$, except for $\q = 0$. At $\q = 0$ or $q = 1$, the transform does not converge as the expression for $\del_{P|Q}(\eps)$ has singularities at these points because numerator $\Eren{\q}{P}{Q} = \int P^\q Q^{1-\q} \dif{\thet} = 1$ when $\q = 0$ or $1$ while denominator becomes zero. Since region of convergence is always a strip in the complex plane $\Com$ parallel to the imaginary line, the R\'enyi divergence exists not just for real orders $\q$, but also for imaginary orders: $\{\q \in \Com \ : \ \mathfrak{Re}(\q) > 1\}$ when $P \ll Q$ and $\{\q \in \Com \ : \ \mathfrak{Re}(\q) < 1 \ \text{and}\ \mathfrak{Re}(\q) \neq 0 \}$ when $Q \ll P$. Therefore, as long as either $P \ll Q$ \emph{or} $Q \ll P$, the region $\mathrm{ROC}_\mathcal{B}\{\del_{P|Q}\}$ is not empty and a characterizing curve $\Ren{\q}{P}{Q}$ exists for the profile $\del_{P|Q}(\eps)$. That is to say,~\autoref{thm:dp_profile_to_rdp} can be used to derive the exact privacy profile $\del_{P|Q}$ using the R\'enyi divergence $\Ren{\q}{P}{Q}$ function. We can do this by substituting $1 - \q = s$ in~\eqref{eqn:dp_profile_to_rdp}, rearranging, and applying the \emph{inverse Laplace transform}~\eqref{eqn:inverse_laplace}, to get the following explicit form:
\begin{align}
	\label{eqn:privacy_prof_inverse_laplace}
	\del_{P|Q}(\eps) &= \mathcal{L}^{-1} \left\{\frac{e^{-s\Ren{1-s}{P}{Q}}}{s(s-1)}\right\}(\eps) = \frac{1}{2\pi i} \lim_{\omega\rightarrow \infty} \int_{\gamma -i \omega}^{\gamma +i \omega} e^{s \eps} \cdot\frac{e^{-s\Ren{1-s}{P}{Q}}}{s(s-1)} ds,
\end{align}
where $\gamma \in \R$ can be \emph{any} real point in $\mathrm{ROC}_\mathcal{B}\{\del_{P|Q}\}$. 

\begin{figure}[t]
\centering
\begin{tikzpicture}
    \begin{axis}[
        width=12cm,
        height=8cm,
        axis x line=middle,
        axis y line=middle,
        xmin=-3, xmax=5,
        ymin=-5, ymax=5,
        xlabel={$\gamma$ (Re)},
        ylabel={$\omega$ (Im)},
	title={Region of Convergence of $\BiLap{\del_{P|Q}}(1-\q)$ for $P \ll Q$ and $Q \ll P$ at R\'enyi order $\q = \gamma + i \omega$},
        xtick={-2, -1, 0, 1, 2, 3, 4},
        ytick={-4, -2, 0, 2, 4},
        legend style={/tikz/every even column/.append  style={column sep=0.5cm}, at={(0.5,-0.05)}, anchor=north, legend columns=-1}
    ]

        \addplot[fill=blue, opacity=0.2, area legend] coordinates {(1, -5) (3, -5) (3, 5) (1, 5) (1, -5)};
        \addlegendentry{$P \ll Q$: $\mathfrak{Re}(q) > 1$}

        \addplot[fill=red, opacity=0.2, area legend] coordinates {(-1, -5) (1, -5) (1, 5) (-1, 5) (-1, -5)};
	\addlegendentry{$Q \ll P$: $\mathfrak{Re}(q) < 1$ except $\mathfrak{Re}(q) = 0$ }

        \node at (axis cs:0,0) {\huge$\times$};
	\addlegendentry{Singularity}
        \node at (axis cs:1,0) {\huge$\times$};

        \addplot[dashed, ultra thick, black] coordinates {(1, -5) (1, 5)};
        \addplot[dashed, ultra thick, black] coordinates {(0, -5) (0, 5)};

        \node[anchor=east] at (axis cs:-1,2.5) {To $-\infty$};
        \draw [arrows = {-Latex[width=10pt, length=10pt]}] (-1,2) -- (-2,2);

        \node[anchor=east] at (axis cs:-1,-1.5) {To $-\infty$};
        \draw [arrows = {-Latex[width=10pt, length=10pt]}] (-1,-2) -- (-2,-2);

        \node[anchor=east] at (axis cs:4,2.5) {To $\infty$};
        \draw [arrows = {-Latex[width=10pt, length=10pt]}] (3,2) -- (4,2);

        \node[anchor=east] at (axis cs:4,-1.5) {To $\infty$};
        \draw [arrows = {-Latex[width=10pt, length=10pt]}] (3,-2) -- (4,-2);

    \end{axis}

\end{tikzpicture}
\end{figure}

\textbf{Case study.} The following example demonstrates an application of the Laplace transform identity in~\eqref{thm:dp_profile_to_rdp}. We first describe the privacy profile of the randomized response mechanism~\citep{kairouz2015composition} and then use~\eqref{eqn:dp_profile_to_rdp} on it to reason about its R\'enyi divergence characteristics.
\begin{theorem}[Privacy profile of randomized response]
	\label{thm:rr_privacy_profile}
	Fix $\eps > 0$ and $\del \in [0, 1]$.
	Let $\M_\mathrm{RR}^{\eps, \del} : \{0, 1\} \rightarrow \{0, 1\} \times \{\bot, \top\} $ be the randomized response mechanism, which has the following output probabilities.
	\begin{equation}
		\label{eqn:rr_defn}
		\M_\mathrm{RR}^{\eps,\del}(0) = \begin{cases}
			(0, \bot) & \text{with probability}\ \del,\\ (0, \top) & \text{with probability}\ \frac{(1-\del)e^{\eps}}{e^\eps+1},\\
			(1, \top) & \text{with probability}\ \frac{(1-\del)}{e^\eps+1}, \\
			(1, \bot) & \text{with probability}\ 0,
		\end{cases}
		\M_\mathrm{RR}^{\eps,\del}(1) = \begin{cases}
			(0, \bot) & \text{with probability}\ 0,\\ 
			(0, \top) & \text{with probability}\ \frac{(1-\del)}{e^\eps+1},\\
			(1, \top) & \text{with probability}\ \frac{(1-\del)e^{\eps}}{e^\eps+1}, \\
			(1, \bot) & \text{with probability}\ \del.
		\end{cases}
	\end{equation}
	For $P = \M_\mathrm{RR}^{\eps,\del}(0)$ and $Q = \M_\mathrm{RR}^{\eps, \del}(1)$, the privacy profiles are
	\begin{equation}
		\label{eqn:rr_privacy_profile}
		\forall t \in \R \ : \ \del_\mathrm{RR} := \del_{P|Q}(t) = \del_{Q|P}(t) =
		\begin{cases}
			\del & \text{if}\ \eps < t, \\
			1 - \frac{e^t+1}{e^\eps + 1}(1-\del) &\text{if} \ -\eps < t \leq \eps, \\ 
			1-e^t(1 - \del) & \text{if} \ t \leq -\eps.
		\end{cases}
	\end{equation}
\end{theorem}
\autoref{thm:rr_privacy_profile} extends \citet[Theorem 2]{balle2020privacy}, which provides the privacy profile for randomized response only in the case where \( \delta = 0 \) and for values \( t \geq 0 \). From~\eqref{eqn:rr_defn}, we can see that neither $P \ll Q$ nor $Q\ll P$ holds for the output distributions of randomized response mechanism when $\del > 0$. As such, the Laplace transform $\BiLap{\del_{P|Q}}(1-\q)$ must not converge for any $\q \in \R$. One can check that when $\q \geq 1$, the transform has a term $\int_{\eps}^\infty e^{(q-1)t} \cdot \delta\mathrm{d}t$ which blows up to $\infty$, and when $\q < 1$, the transform has a term $\int_{-\infty}^{-\eps} e^{(\q-1)t} \cdot (1 - e^t(1-\del)) \mathrm{d}t$ that blows up to $-\infty$. Hence, the R\'enyi divergence of the privacy profile~\eqref{eqn:rr_privacy_profile} cannot be defined for any order $\q \in \R$ when $\del > 0$.

When $\del = 0$, note that both $P \ll Q$ and $Q \ll P$. Therefore, the Laplace transform $\del_{P|Q}(\eps)$ must exist for all $\q \in \R \setminus \{0, 1\}$. The following theorem shows the resulting R\'enyi divergence curve, derived by computing the Laplace transform.
\begin{theorem}[R\'enyi DP of $(\eps, 0)$-Randomized Response]
	\label{thm:rdp_of_rr}
	For any $\eps > 0$ and $\del = 0$, the output distributions of randomized response mechanism in~\autoref{thm:rr_privacy_profile} exhibit a R\'enyi divergence
	\begin{equation}
		\label{eqn:rdp_of_rr}
		\forall \q \in \Com \ \text{s.t.} \ \mathfrak{Re}(\q) \not\in\{0, 1\} : \Ren{\q}{P}{Q} = \frac{1}{\q-1} \log \left(\frac{e^\eps}{1+e^{\eps}} e^{-\q\eps} + \frac{1}{1+e^{\eps}} e^{\q\eps}\right).
	\end{equation}
\end{theorem}
\autoref{thm:rdp_of_rr} generalizes \citet[Proposition 5]{mironov2017renyi}, which gives the R\'enyi divergence of randomized response only for real orders \( q > 1 \). In the following section, we elaborate on the significance of complex orders in R\'enyi divergence.
\autoref{fig:dp_rdp_dominance} visualizes the privacy profile $\del_{P|Q}$, its Laplace transform $\BiLap{\del_{P|Q}}$, and the corresponding R\'enyi divergence $\Ren{\q}{P}{Q}$ of this randomized response mechanism, and compares it with that of Gaussian mechanism~(cf.~\autoref{thm:gaussian_privacy}).

\subsection{Dominance: R\'enyi Divergence vs. Privacy Profile}
\label{ssec:dominance}

\begin{figure}[t]
	\centering
	\includegraphics[width=\linewidth]{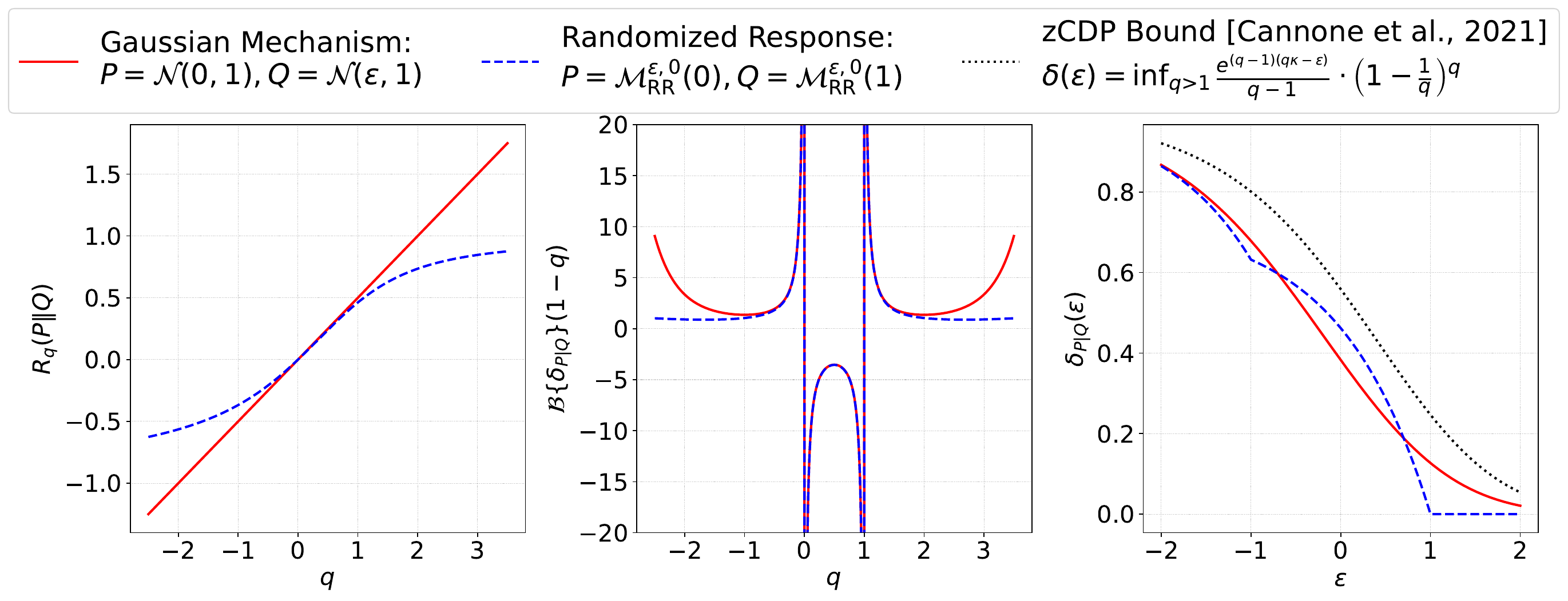}
	\caption{\label{fig:dp_rdp_dominance} Comparison between the indistinguishability characteristic of Gaussian mechanism of~\autoref{thm:gaussian_privacy} ($\kappa = \Vert\mu - \mu'\Vert^2/2\sigma^2 = \eps^2 / 2$) and randomized response mechanism defined in~\autoref{thm:rr_privacy_profile} (with $\del = 0$). This figure visualizes the singularities at $\q = 0$ and $\q = 1$ that exists for the Laplace transform $\BiLap{\del_{P|Q}}(1-q)$ disappears for R\'enyi divergence $\Ren{\q}{P}{Q}$, which is an effect of the replica trick unfolding. This figure also demonstrates that neither the dominance $\Ren{\q}{P_1}{Q_1} \leq \Ren{\q}{P_2}{Q_2}$ for all $\q > 1$, nor the dominance $\BiLap{\del_{P_1|Q_1}}(1-\q) \leq \BiLap{\del_{P_2|Q_2}}(1-\q)$ for all $\q \in \R \setminus \{0, 1\}$ is enough to bound $\del_{P_1|Q_1}(\eps) \leq \del_{P_2|Q_2}(\eps)$ at all $\eps \in \R$. 
	Additionally, the black dotted line in the rightmost plot shows that even the tightest\protect\footnotemark conversion on the R\'enyi curve considering only real orders $\q > 1$ fails to characterize its own privacy profile.
	}
\end{figure}

\footnotetext{The tightest conversion by~\citet{asoodeh2020better} lacks a closed-form expression and is challenging to approximate numerically in a stable way. Therefore, we compare with~\citet[Corollary 13]{canonne2020discrete}, which yields  similar values.}

Differential privacy is a study of distributional divergence between output distributions $P$ and $Q$ not just for a pair of neighboring inputs $\D, \D'$ but across \emph{all} neighboring inputs. When considering functional notions of DP, comparing indistinguishability characteristics of two output-distribution pairs, say $(P_1, Q_1)$ and $(P_2, Q_2)$, requires a notion of \emph{dominance}. \citet{zhu2022optimal} define a \emph{dominating pair of distribution} specific to a mechanism $\M$ as a pair of distributions $P, Q$ such that 
\begin{equation}
	\forall \eps \in \R \ : \ \sup_{\D \simeq \D'} \del_{\M(\D)|\M(\D')}(\eps) \leq \del_{P|Q}(\eps).
\end{equation}
Following the definition by~\citet{zhu2022optimal}, we define the following notions of dominance.
\begin{definition}[Dominance for Distribution Pairs]
	\label{defn:dominance}
	We say that the distribution pair \( (P_2, Q_2) \) dominates \( (P_1, Q_1) \) in privacy profile (denote as $(P_1, Q_1) \preceq_\del (P_2, Q_2)$) if
	\begin{equation}
		\forall \eps \in \R \ : \ \del_{P_1|Q_1}(\eps) \leq \del_{P_2|Q_2}(\eps).
	\end{equation}
	And, we say \( (P_2, Q_2) \) dominates \( (P_1, Q_1) \) in R\'enyi divergence (denote as $(P_1, Q_1) \preceq_R (P_2, Q_2)$) if 
	\begin{equation}
		\forall \q > 1 \ : \Ren{\q}{P_1}{Q_1} \leq \Ren{\q}{P_2}{Q_2}.
	\end{equation}
\end{definition}
The following theorem reveals the surprising fact that, while the R\'enyi divergence curve and privacy profile curve are equivalent for all absolutely continuous distribution pairs, this equivalence does not imply an identical dominance ordering across absolutely continuous distribution pairs.
\begin{theorem}
	\label{thm:dominance}
	Consider two absolutely continuous distribution pairs $(P_1, Q_1)$ and $(P_2, Q_2)$. Then
	\begin{equation}
		(P_1, Q_1) \preceq_\del (P_2, Q_2) \implies (P_1, Q_1) \preceq_R (P_2, Q_2).
	\end{equation}
	However, the opposite direction does not hold:
	\begin{equation}
		(P_1, Q_1) \preceq_R (P_2, Q_2) \;\not\!\!\!\implies (P_1, Q_1) \preceq_\del (P_2, Q_2).
	\end{equation}
\end{theorem}

\begin{proof}
	First part directly follows from~\eqref{eqn:rdp_as_laplace} by noting that
	\begin{equation}
		\BiLap{\del_{P_1|Q_1}}(1-\q) = \int_{-\infty}^\infty e^{-(1-\q)t} \cdot \del_{P_1|Q_1}(t)\dif{t} \leq \int_{-\infty}^\infty e^{-(1-\q)t} \cdot \del_{P_2|Q_2}(t)\dif{t} = \BiLap{\del_{P_2|Q_2}}(1-\q)
	\end{equation}
	and that for $\q > 1$, $\q (\q - 1) > 0$. 

	The proof of the second statement is based on the example by~\citet{zhu2022optimal} comparing the R\'enyi DP curve and privacy profile of Gaussian mechanism described in~\autoref{thm:gaussian_privacy} (denote with $(P_2, Q_2)$) with that of randomized response mechanism, defined in~\autoref{thm:rr_privacy_profile} (denote with $(P_1, Q_1)$). For any $\eps > 0$, we set $\kappa = \eps^2/2$ in the Gaussian mechanism where $(\eps, 0)$ is the parameter of the randomized response and compare them in~\autoref{fig:dp_rdp_dominance}. From the leftmost plot, we can see that $(P_1, Q_1) \preceq_R (P_2, Q_2)$ holds.\footnote{\citet[Proposition B.7 (a)]{dong2019gaussian} show that $\Ren{\q}{\mathcal{M}_\mathrm{RR}^{\eps,0}(0)}{\mathcal{M}_\mathrm{RR}^{\eps,\del}(1)} \leq \Ren{\q}{\mathcal{N}(0,1)}{\mathcal{N}(\eps,1)}$ for all $\q > 1$, but the bound actually holds for all $\q > 0$.} However, note that in the rightmost plot that their privacy profiles cross one another. Hence, the dominance $(P_1, Q_1) \preceq_\del (P_2, Q_2)$ does not hold. 
\end{proof}

\citet{zhu2022optimal} point out that it is troubling that identifying a dominating pair of distributions in R\'enyi divergence does not guarantee a corresponding dominance in privacy profiles. Although the R\'enyi divergence curve $\Ren{\q}{P}{Q}$ is an equivalent representation of the privacy profile $\del_{P|Q}(\eps)$, converting the $\Ren{\q}{P}{Q}$ curve into a tight upper bound for privacy profile will \emph{always} introduce a gap. This gap can be substantial, as shown in~\autoref{fig:dp_rdp_dominance}, where we compare the (nearly) tight upper bound (black dotted line) derived for the privacy profile from the R\'enyi divergence curve (red line on the left) of the Gaussian mechanism with the actual privacy profile of this Gaussian mechanism.

\textbf{Role of complex orders.} While recognizing the fundamental gap between the tightest achievable bound on the privacy profile derived from a bound on the R\'enyi divergence curve and the privacy profile itself, no such tight bound has yet been established. The best available bounds rely on taking an infimum over pointwise conversions from R\'enyi divergence at a single order $\q > 1$ to \((\eps, \del)\)-DP~\citep{canonne2020discrete, bun2016concentrated}. However, since pointwise guarantees offer a lossy characterization of indistinguishability, taking an infimum over privacy profiles implied by these guarantees is unlikely to achieve tight upper bounds. Notably, from the Inverse Laplace Transform expression~\eqref{eqn:privacy_prof_inverse_laplace} for privacy profiles, we observe that the value of \( \del_{P|Q} \) at any \( \epsilon \) depends on the behavior of the Rényi divergence \( \Ren{\q}{P}{Q} \) along the complex line \( \mathfrak{Re}(1 - \q) = \gamma \) for any $\gamma \in \R \setminus \{0, 1\}$, and not along the real line. In fact, the choice of $\gamma$ on $\R$ does not matter at all, as long as it lies in the $\mathrm{ROC}_\mathcal{B}\{\del_{P|Q}\}$. Thus, we believe that establishing a tight functional conversion will require examining how the Rényi divergence \emph{curls as we move the order $\q$ along the complex line} $\gamma + i \omega$.

\section{Exactly-Tight Composition Theorems}
\label{sec:composition}

Unlike functional DP guarantees, point guarantees like $(\eps, \del)$-DP or $(\q, \rhO)$-R\'enyi DP are a lossy characterization of the indistinguishability between two distributions $P$ and $Q$. Despite this, using them for reporting or certifying an algorithm's worst-case privacy is generally acceptable as the privacy protection they guarantee, although conservative, is adequate.
The main issue arises when attempting to compose point DP guarantees, as the quantification loss often compounds drastically, resulting in a significant overestimation of the actual privacy protection provided by an algorithm. Consider the $k$-fold (non-adaptive) composition of a one-dimensional Gaussian mechanism with $L_2$ sensitivity 1 and noise variance $\sigma^2 = 1$. For individual 1D Gaussian distributions $P = \mathcal{N}(0, 1)$ and $Q = \mathcal{N}(1, 1)$, the privacy profile is of the order $\del_{P|Q}(\eps) = O(e^{-\eps^2/2})$ (cf.~\autoref{thm:gaussian_privacy}), indicating that the mechanism satisfies a point guarantee of $(O(\sqrt{\log(1/\del)}), \del)$-DP for any $\del \in (0, 1]$. When we extend this to $k$-fold non-adaptive self-composition, the resulting output distribution is a $k$-dimensional Gaussian, specifically $P^{\otimes k} = \mathcal{N}(\vec{0}, I_k)$ and $Q^{\otimes k} = \mathcal{N}(\vec{1}, I_k)$. The privacy profile of this composition is of the asymptotic order $\del_{P^{\otimes k}|Q^{\otimes k}}(\eps) = O(e^{-\eps^2/2k})$, yielding a point guarantee of $(O(\sqrt{k \log (1/\del)}), \del)$-DP. However, if we attempt to compose the individual $(O(\sqrt{\log(1/\del)}), \del)$-DP point guarantees for each Gaussian mechanism, even with the optimal composition theorem for $(\eps, \del)$-DP point guarantees from \citet{kairouz2015composition}, the best achievable guarantee is $(O(\sqrt{k} \log(1/\del)), (k+1)\del)$-DP. This result is significantly more pessimistic than the true privacy profile---a factor of $O(\sqrt{\log (k / \del)})$ in the $\eps$ for the same $\del$.

Functional notions of DP effectively address this problem by capturing the indistinguishability between any two distributions $P$ and $Q$ accurately \citep{dong2019gaussian, koskela2020computing}. However, R\'enyi DP (as a function of order $\q$) remains the only functional DP notion with an \emph{exactly tight} composition theorem (i.e., matching even the constants) that can accommodate \emph{adaptively chosen heterogeneous mechanisms}.\footnote{Adaptive means each mechanism’s output may depend on previous outputs; heterogeneous means the mechanisms need not be identical.} 
\begin{remark}
	The previous statement requires some justifications. We note that the PLD formalism for composition is limited to non-adaptive mechanisms \citep{sommer2018privacy, koskela2020computing, gopi2021numerical}. On the other hand for $f$-DP, there is no general composition formula for arbitrary trade-off curves $f_{P_i|Q_i}$. Instead, \citet{dong2019gaussian} provides an explicit composition operator expression applicable only to Gaussian trade-off functions (Corollary 3.3). Furthermore, \citet{zhu2022optimal}'s characteristic function formalism~\eqref{eqn:zhu_char} appears equivalent to R\'enyi divergence, as we have learned that the order $\q$ in \eqref{eqn:rdp_as_laplace} can assume imaginary values.
\end{remark}

In the following theorem, we provide an \emph{exactly tight} composition result that applies to arbitrary privacy profiles $\del_{P_1|Q_1}$ and $\del_{P_2|Q_2}$. Later we also extend this theorem to handle adaptivity.

\begin{theorem}[Exactly Tight Composition of Privacy Profiles]
	\label{thm:compose_privacy_profile}
	If $P= P_1 \times P_2$ and $Q = Q_1 \times Q_2$ are two product distributions on $\OO_1 \times \OO_2$ such that $(P_1, Q_1)$ and $(P_2, Q_2)$ are absolutely continuous at least in one direction, then
	\begin{equation}
		\label{eqn:compose_privacy_profile}
		\del_{P|Q}(\eps) = \left(\del_{P_1|Q_1} \convolve \left(\ddot\del_{P_2|Q_2} - \dot\del_{P_2|Q_2}\right) \right)(\eps) = \int_{-\infty}^\infty \del_{P_1|Q_1}(\eps - \tau) \cdot \left(\ddot\del_{P_2|Q_2}(\tau) - \dot\del_{P_2|Q_2}(\tau)\right) \dif{\tau},
	\end{equation}
	where $\dot\del_{P_2|Q_2}$ and $\ddot\del_{P_2|Q_2}$ are the first and second order gradient functions of $\del_{P_2|Q_2}$.
\end{theorem}
\begin{proof}
	Let's consider the random variables $\Z \sim \privloss{P}{Q}$, $\Z_1 \sim \privloss{P_1}{Q_1}$, and $\Z_2 \sim \privloss{P_2}{Q_2}$. For a pair $(\Thet_1, \Thet_2) \sim P$, the privacy loss random variable $\Z$ is given by $\priv{P}{Q}(\Thet_1, \Thet_2)$, which simplifies to:
	\begin{equation}
		\log \frac{P_1(\Thet_1) P_2(\Thet_2)}{Q_1(\Thet_1)Q_2(\Thet_2)} = \priv{P_1}{Q_1}(\Thet_1) + \priv{P_2}{Q_2}(\Thet_2).
	\end{equation}
	This decomposition implies that $\Z$ can be expressed as the sum of $\Z_1$ and $\Z_2$, i.e., $\Z = \Z_1 + \Z_2$. Consequently, the probability density of $\Z$, $f_\Z$, is the convolution of $f_{\Z_1}$ and $f_{\Z_2}$:
	\begin{equation}
		f_\Z(t) = \int_{-\infty}^\infty f_{\Z_1}(\tau) f_{\Z_2}(t-\tau)d \tau = (f_{\Z_1} \convolve f_{\Z_2})(t).
	\end{equation}
	Invoking Theorem~\ref{thm:dp_as_laplace}, we then obtain the Laplace transform of $f_\Z$ at $(1-\q)$:
	\begin{align}
		\BiLap{f_\Z(t)}(1-\q) \overset{\eqref{eq:Convolution}}{=} \BiLap{f_{\Z_1}(t)}(1-\q) \cdot \BiLap{f_{\Z_2}(t)}(1-\q).
	\end{align}
	From Definition~\ref{def:rdp}, this directly implies that the R\'enyi divergence of order $\q$ for the pair $(P, Q)$ is the sum of the R\'enyi divergences for the pairs $(P_1, Q_1)$ and $(P_2, Q_2)$ as
	\begin{equation}
		\Ren{\q}{P}{Q} = \Ren{\q}{P_1}{Q_1} + \Ren{\q}{P_2}{Q_2}.
	\end{equation}
	Now suppose $s = 1-\q < 0$. Thanks to absolute continuity, we can express the R\'enyi divergences in terms of the Laplace transform using Theorem~\ref{thm:dp_profile_to_rdp} as follows, cancelling out the common terms:
	\begin{equation}
		s(s-1) \BiLap{\del_{P|Q}(t)}(s) = s(s-1) \BiLap{\del_{P_1|Q_1}(t)}(s)  \cdot s(s-1) \BiLap{\del_{P_2|Q_2}(t)}(s).
	\end{equation}
	Since $\q = 1 - s$ cannot be $0$ or $1$, we can divide by $s(s-1)$ on both sides:
	\begin{align}
		\BiLap{\del_{P|Q}(t)}(s)&= \BiLap{\del_{P_1|Q_1}(t)}(s)  \cdot s(s-1)\BiLap{\del_{P_2|Q_2}(t)}(s) \\
					&= \BiLap{\del_{P_1|Q_1}(t)}(s)  \cdot \left(s^2 \BiLap{\del_{P_2|Q_2}(t)}(s) - s\BiLap{\del_{P_2|Q_2}(t)}(s)\right) \\
					&\overset{\eqref{eq:Differentiation}}{=} \BiLap{\del_{P_1|Q_1}(t)}(s) \cdot \left(\BiLap{\ddot\del_{P_2|Q_2}(t)}(s) -\BiLap{\dot\del_{P_2|Q_2}(t)}(s)\right) \\
					&\overset{\eqref{eq:Linearity}}{=} \BiLap{\del_{P_1|Q_1}(t)}(s) \cdot \BiLap{\ddot\del_{P_2|Q_2}(t) - \dot\del_{P_2|Q_2}(t)}(s) \\
				     &\overset{\eqref{eq:Convolution}}{=} \BiLap{\left(\del_{P_1|Q_1} \circledast \left(\ddot\del_{P_2|Q_2} - \dot\del_{P_2|Q_2}\right)\right)(t)}(s).
	\end{align}
	Hence, from the uniqueness of Laplace transform, we get 
	\begin{equation}
		\forall \eps \in \R, \quad \del_{P|Q}(\eps) = \left(\del_{P_1|Q_1} \circledast \left(\ddot\del_{P_2|Q_2} -\dot\del_{P_2|Q_2}\right) \right)(\eps).
	\end{equation}
\end{proof}
\autoref{thm:compose_privacy_profile} provides an \emph{exactly tight} composition theorem---not only because the terms in~\eqref{eqn:compose_privacy_profile} are equal, but also because the privacy profiles \(\del_{P_1|Q_1}\) and \(\del_{P_2|Q_2}\) precisely capture the indistinguishability of their respective distributions. This theorem mirrors the composition property of R\'enyi divergence but works in the time domain \(\eps\) instead of the frequency domain \(\q\). It assumes absolute continuity in at least one direction (\(P_i \ll Q_i\) or \(Q_i \ll P_i\)) for both \(i = 1\) and \(2\). 

Interestingly however, our result in \eqref{eqn:compose_privacy_profile} appears to hold even when absolute continuity fails in either direction. We will see an example of this in the next section where we apply~\eqref{eqn:compose_privacy_profile} to compose the privacy profile of randomized mechanisms $\del_\mathrm{RR}^{\eps_1,\del_1} \otimes \del_\mathrm{RR}^{\eps_2,\del_2}$ and get a tight expression for $\del_1, \del_2 > 0$ without running into a singularity. This happens because even when the Laplace transform is undefined everywhere, the frequency-domain manipulations performed on it still correspond to valid manipulation steps in the time domain. Since our main interest lies in the time domain function---the composed privacy profile---taking advantage of this flexibility proves beneficial. To emphasize the significance of this, \autoref{thm:compose_privacy_profile} allows us to tightly derive composed privacy profiles with the same exact-tightness as R\'enyi DP even for mechanisms that do not satisfy R\'enyi-DP for any order $\q \in \R\setminus \{0, 1\}$, opening exciting possibilities beyond the limitations of R\'enyi DP. 

\begin{remark}
	Formally proving why dropping the absolute continuity assumption in~\autoref{thm:compose_privacy_profile} does not compromise the validity of~\eqref{eqn:compose_privacy_profile} appears to be a challenging yet intriguing problem. Our efforts to resolve this suggests that a mathematical understanding of the number $0^i$ is necessary.
\end{remark}

\textbf{Adaptive Composition.} For two mechanisms $\M_1$ and $\M_2$, if mechanism $\M_2$ sees the output from $\M_1$, then the output distribution of $\M_1$ and $\M_2$ are no longer independent. \autoref{thm:compose_privacy_profile} can still be applied as long as their exists a distribution pair $P_2, Q_2$ that \emph{dominates the privacy profile} of output distribution pairs $\M_2(\D, \thet)$ and $\M_2(\D', \thet)$ across all $\thet$ for a given dataset pair $\D \simeq \D'$, which is a reasonable assumption for adaptively compositing functional guarantees for DP.\footnote{Adaptive composition for point DP guarantees requires that for fixed $\D \simeq \D'$, conditioned on any observation from $\M_1$, the point DP guarantee holds for $\M_2$. Adaptively composing DP curves needs a stronger assumption that conditioned on any output of $\M_1$, the privacy profile of $\M_2$ lies below a worst-case DP curve.} This is possible due to the following result.
\begin{lemma}[{\citet[Theorem 27]{zhu2022optimal}}]
	\label{lem:adaptive_profile}
	Let $P(x,y) = P_1(x)\cdot P^x_2(y)$ and $Q(x,y) = Q_1(x)\cdot Q^x_2(y)$ be two joint distributions on $\OO_1 \times \OO_2$. Then for any distributions $P_2$ and $Q_2$ on $\OO_2$ such that $\del_{P^x_2|Q^x_2}(\eps) \leq \del_{P_2|Q_2}(\eps)$ for all $\eps \in \R$ and $x \in \OO_1$, we have $\del_{P|Q}(\eps) \leq \del_{P_1 \times P_2 |Q_1 \times Q_2}(\eps)$.
\end{lemma}

\subsection{Tight Composition for $(\eps, \del)$-DP} 
\label{ssec:point_dp_composition}

In this section, we use~\autoref{thm:compose_privacy_profile} to prove an exactly-tight composition theorem for adaptive composition of a sequence of $(\eps_i, \del_i)$-DP mechanisms. We begin with a variant of~\citet{kairouz2015composition}'s result that the privacy profile $\del_{P|Q}(\eps)$ under an $(\eps, \del)$-DP point guarantee is dominated by the privacy profile of randomized response mechanism.
\begin{theorem}[{Dominating Privacy Profile under $(\eps, \del)$-DP~\citep{kairouz2015composition}}]
	\label{thm:rr_dominates_approxdp}
	Fix $\eps \geq 0$ and $\del \in [0, 1]$. Suppose distributions $P$ and $Q$ over $\OO$ satisfy $(\eps, \del)$-differential privacy. Then,
	\begin{equation}
		\label{eqn:rr_dominates_approxdp}
		\forall \eps \in \R \ : \ \del_{P|Q}(\eps) \leq \del_\mathrm{RR}(\eps) \quad \text{and}\quad \del_{Q|P}(\eps) \leq \del_\mathrm{RR}(\eps),
	\end{equation}
	where $\del_\mathrm{RR}(t)$ is the privacy profile of the randomized response mechanism $\M_\mathrm{RR}^{\eps, \del}$.
\end{theorem}

\citet{kairouz2015composition} do not express their notion of \emph{dominance} in the same way as we do in Definition~\ref{defn:dominance}---they say distribution pair $(P_1, Q_1)$ dominates $(P_2, Q_2)$ if their trade-off curves satisfy
\begin{equation}
	\forall \alpha \in [0,1] \ : \ f_{P_1|Q_1}(\alpha) \geq f_{P_2|Q_2}(\alpha).
\end{equation}
Nonetheless, our notion of dominance and their notion is equivalent, which was established by~\citep{dong2019gaussian} by showing that the privacy profile $\del_{P_i|Q_i}$ and the corresponding trade-off curve $f_{P_i|Q_i}$ are \emph{primal} and \emph{dual} with respect to Frenchel duality. We also provide a direct proof of~\autoref{thm:rr_dominates_approxdp} in the Appendix~\ref{app:composition}. 

As a side note, observe that combining~\autoref{thm:rr_dominates_approxdp} with~\autoref{thm:rdp_of_rr} gives a tight R\'enyi DP guarantee for a pure $\eps$-DP mechanism, which has recently attracted interest~\citep{differentialprivacyTightZCDP}.
\begin{equation}
	\del_{P|Q}(\eps) = \del_{Q|P}(\eps) = 0  \implies \forall \q > 1 \ : \ \Ren{\q}{P}{Q} \leq \frac{1}{\q - 1}\log\left(\frac{e^\eps}{e^\eps + 1} e^{-\q\eps} + \frac{1}{e^\eps + 1}e^{\q \eps}\right).
\end{equation}

Following the objective of this section, we use our~\autoref{thm:compose_privacy_profile} on the above worst-case privacy profile under $(\eps, \del)$-DP point guarantees, resulting in an exactly-tight composition guarantee as stated below.
\begin{theorem}[{Tight Composition for $(\eps, \del)$-DP}]
	\label{thm:approxdp_composition_hetro}
	For any $\eps_i \geq 0$, $\del_i \in [0,1]$ for $i \in \{1, \cdots, k\}$, the $k$-fold composition of $(\eps_i, \del_i)$-differentially private mechanisms satisfies $(\eps, \del^{\otimes k}(\eps))$-DP for all $\eps$, defined recursively as
	\begin{equation}
		\label{eqn:approxdp_composition_hetro}
		\forall t \in \R \ : \ \del^{\otimes l}(t) = \del_l + \frac{(1-\del_l)}{e^{\eps_l} + 1}\left[ e^{\eps_l} \cdot \del^{\otimes l-1}(t -\eps_l) + \del^{\otimes l-1}(t +\eps_l)\right],
	\end{equation}
	with $\del^{\otimes 0}(t) = [1 - e^t]_+$.
\end{theorem}
\begin{figure}[t]
	\centering
	\includegraphics[width=\linewidth]{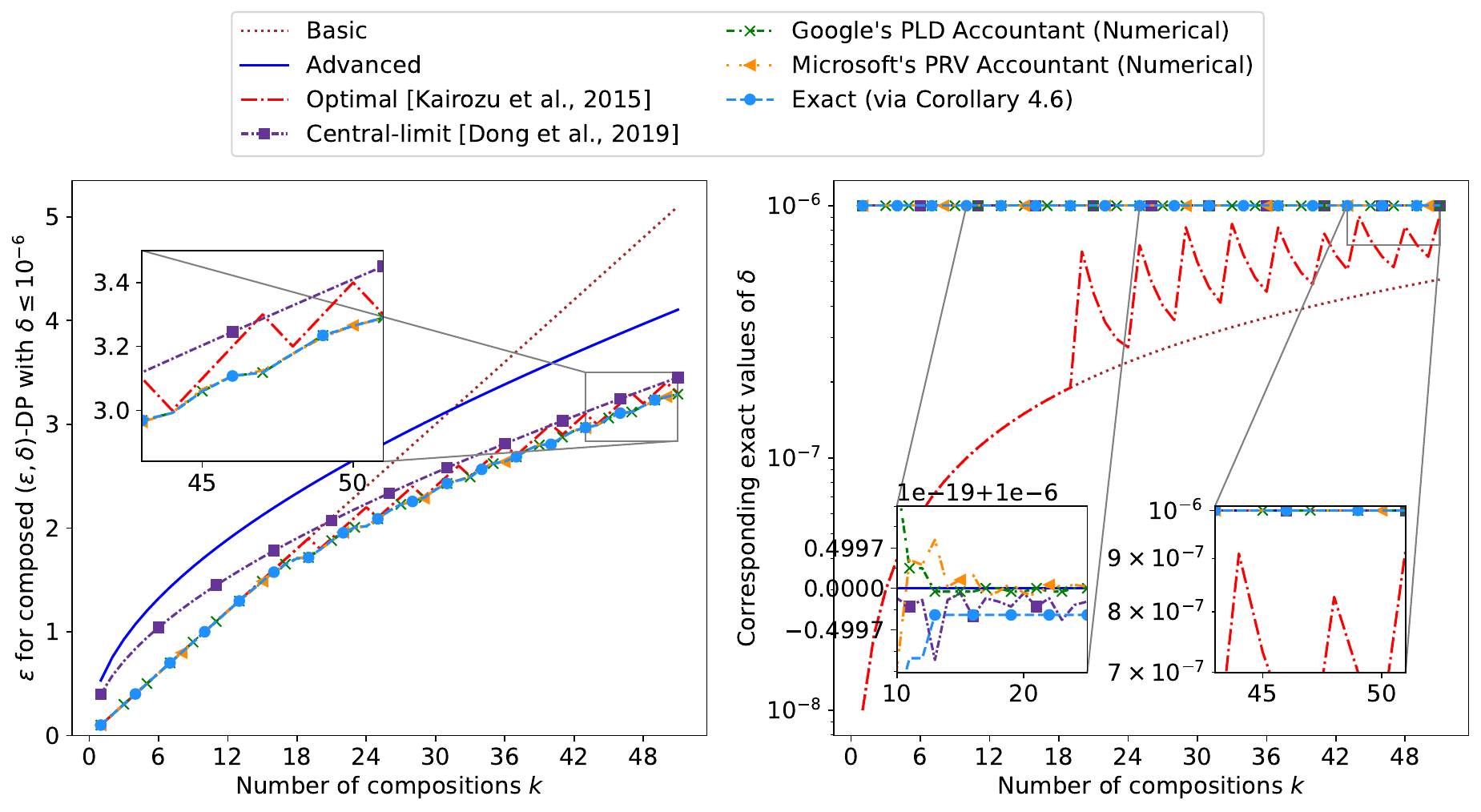}
	\caption[caption of figure]{\label{fig:Laplace_composition}
		Comparison of $(\eps, \del)$-DP bounds from various composition theorems for $k$-fold composition of a $(0.1, 10^{-8})$-DP point guarantee, with the budget constraint $\del < 10^{-6}$. The spikes in the right plot, showing exact $\del < 10^{-6}$ values from \citet[Theorem 3.3]{kairouz2015composition}, occur because, out of a set of $\lfloor k/2 \rfloor$ DP point guarantees by their result, we select the smallest $\eps$ corresponding to the largest $\del < 10^{-6}$ in the set, which fluctuates as $k$ increases.
	}
\end{figure}

\autoref{thm:approxdp_composition_hetro} introduces a convenient recursive method for computing compositions of heterogeneous DP guarantees. While this bound matches each of the $\lfloor k \rfloor$ discrete $(\eps, \del)$-DP values given by the optimal composition theorem from \citet[Theorem 3.3]{kairouz2015composition}, our result offers a continuous curve over all \(\eps \in \R\). Consequently, for a given budget on $\del$, our bound provides a tighter limit on $\eps$ than that of \citep{kairouz2015composition}, as shown in~\autoref{fig:Laplace_composition}. Furthermore, the recursion reduces to the following exact expression when composing homogeneous DP guarantees.
\begin{corollary}
	\label{corr:approxdp_composition_homo}
	For any $\eps \geq 0$, $\del \in [0,1]$, the $k$-fold composition of $(\eps, \del)$-DP mechanisms satisfies $(\eps, \del^{\otimes k}(\eps))$-DP for all $\eps$, where
	\begin{equation}
		\label{eqn:approxdp_composition_homp}
		\forall t \in \R \  : \ \del^{\otimes k}(t) = 1 - (1-\del)^k \left( 1- \expec{Y \leftarrow  \mathrm{Binomial}\left(k, \frac{e^\eps}{1+e^\eps}\right)}{1 - e^{t-\eps \cdot (2Y - k)}}_+\right).
	\end{equation}
\end{corollary}
\autoref{fig:Laplace_composition} illustrates the enhanced privacy quantification achieved by our bound for \(k\)-fold composition of $(\eps, \del)$-DP guarantees. We also compare our results with Google's PLDAccountant~\citep{doroshenko2022connect} and Microsoft's PRVAccountant~\citep{gopi2021numerical}, which utilize the Discrete Fast Fourier Transform. This comparison shows that our analytical approach closely matches the values approximated by these numerical methods. Additionally, through \autoref{fig:composition_numerical_compare} in the appendix, we show that these numerical methods can be unstable at edge case values and yield non-negligible gaps in their approximation.

\section{Asymmetry and DP Notion Equivalences}
\label{sec:subsampling}

An important characteristic of functional notions of DP is that they can be \emph{asymmetric}, in the sense that switching $P \leftrightarrow Q$ may yield a different curve $\del_{Q|P}$ than the original $\del_{P|Q}$. In context of a mechanism $\M$, such an asymmetry between its output distribution $P = \M(\D)$ and $Q = \M(\D')$ means that a sample $\Thet \sim P$ might reveal more (or less) information that it came from $\D$ than a sample $\Thet' \sim Q$ reveals about coming from $\D'$. Since \( \D \) and \( \D' \) are neighboring datasets differing by the presence or absence of a single record, this asymmetry in indistinguishability means that an attacker might have an easier time trying to detect the presence of a record from the output of \( \mathcal{M} \) than to detect its absence. In other words, optimal hypothesis tests would experience a skew in the trade-off between their false positive and false negative rates.

Such skewness often arises due to \emph{subsampling}, which is a heavily used technique to boost an algorithm's intrinsic privacy properties~\citep{balle2020privacy,bassily2014private,abadi2016deep,wang2015privacy,balle2018privacy}. For instance, Poisson subsampling in the context of the \emph{add or remove} relationship between neighboring datasets, which is commonly used in DP-SGD~\citep{abadi2016deep} algorithm, skews the privacy profile of a base mechanism, as illustrated in the following example.

\textbf{Effect of Poisson Subsampling on $\del_{P|Q}$.} Without loss of generality, assume datasets $\D\simeq\D'$ are such that the record at index $i$ is present in $\D$ but empty in $\D'$, i.e., $\D[i] \neq \D'[i] = \bot$. If we randomly filter the records using an iid selection mask $U \sim \mathrm{Bernoulli}(\lambda)^{\otimes \n}$, the subsampled datasets $\D_U$ and $\D'_U$ are defined as follows for all $i \in [n]$:
\begin{equation}
	\D_U[i] \eqdef \begin{cases} \D[i] &\text{ if }  U_i = 1 \\ \bot  &\text{ otherwise} \end{cases}  \quad \text{and}\quad  \D'_U[i] \eqdef \begin{cases} \D'[i] &\text{ if }  U_i = 1 \\ \bot  &\text{ otherwise} \end{cases}.
\end{equation}
The distributions $P$ and $Q$ of the outputs $\M(\D_U)$ and $\M(\D'_U)$ for any algorithm $\M$ will be identical with probability $1 - \lambda$, which amplifies privacy considerably. More precisely, let $P_\mathrm{IN}$ and $Q_\mathrm{IN}$ be the distributions of $\M(\D_U)$ and $\M(\D'_U)$ conditioned on $i \in U$, and $P_\mathrm{OUT}$ and $Q_\mathrm{OUT}$ be the distributions conditioned on $i \not\in U$. For any event $S \subseteq \OO$, we have:
\begin{align}
	P(S) &= \Pr[i \not\in U]\cdot P_\mathrm{OUT}(S) + \Pr[i \in U] \cdot P_\mathrm{IN}(S)  \nonumber \\
	     &= (1-\lambda) \cdot Q_\mathrm{OUT}(S) + \lambda \cdot P_\mathrm{IN}(S) \label{eqn:subsampling1} \\
	     &= (1-\lambda) \cdot Q_\mathrm{IN}(S) + \lambda \cdot P_\mathrm{IN}(S). \label{eqn:subsampling2}
\end{align}
Equation~\eqref{eqn:subsampling1} holds because if $i \not\in U$, then $\D_U = \D'_U$, and equation~\eqref{eqn:subsampling2} holds because the $i$th record in $\D'$ is empty, so conditioning on $i \in U$ or $i \not\in U$ does not affect the output distribution, i.e., $Q_\mathrm{IN} = Q_\mathrm{OUT} = Q$. Using this fact, we the following theorem shows the \emph{exact effect} Poisson subsampling has on the privacy profile.

\begin{theorem}[Poisson subsampling on add/remove neighbours]
	\label{thm:subsampling}
	Let $0 < \lambda \leq 1$. For any distributions $P$, $Q$, $P_\mathrm{IN}$ and $Q_\mathrm{IN}$ such that $P = \lambda P_\mathrm{IN} + (1-\lambda) Q_\mathrm{IN}$ and $Q = Q_\mathrm{IN}$, 
	\begin{align}
		\label{eqn:subsampling}
		\del_{P| Q}(\eps) = \begin{cases} \lambda \del_{P_\mathrm{IN}|Q_\mathrm{IN}}(\log(1 + (e^{\eps} - 1)/\lambda)) & \text{if } \eps > \log(1-\lambda), \\1 - e^\eps &\text{otherwise}. \end{cases}
	\end{align}
\end{theorem}

\begin{figure}[!htb]
	\centering
	\includegraphics[width=\linewidth]{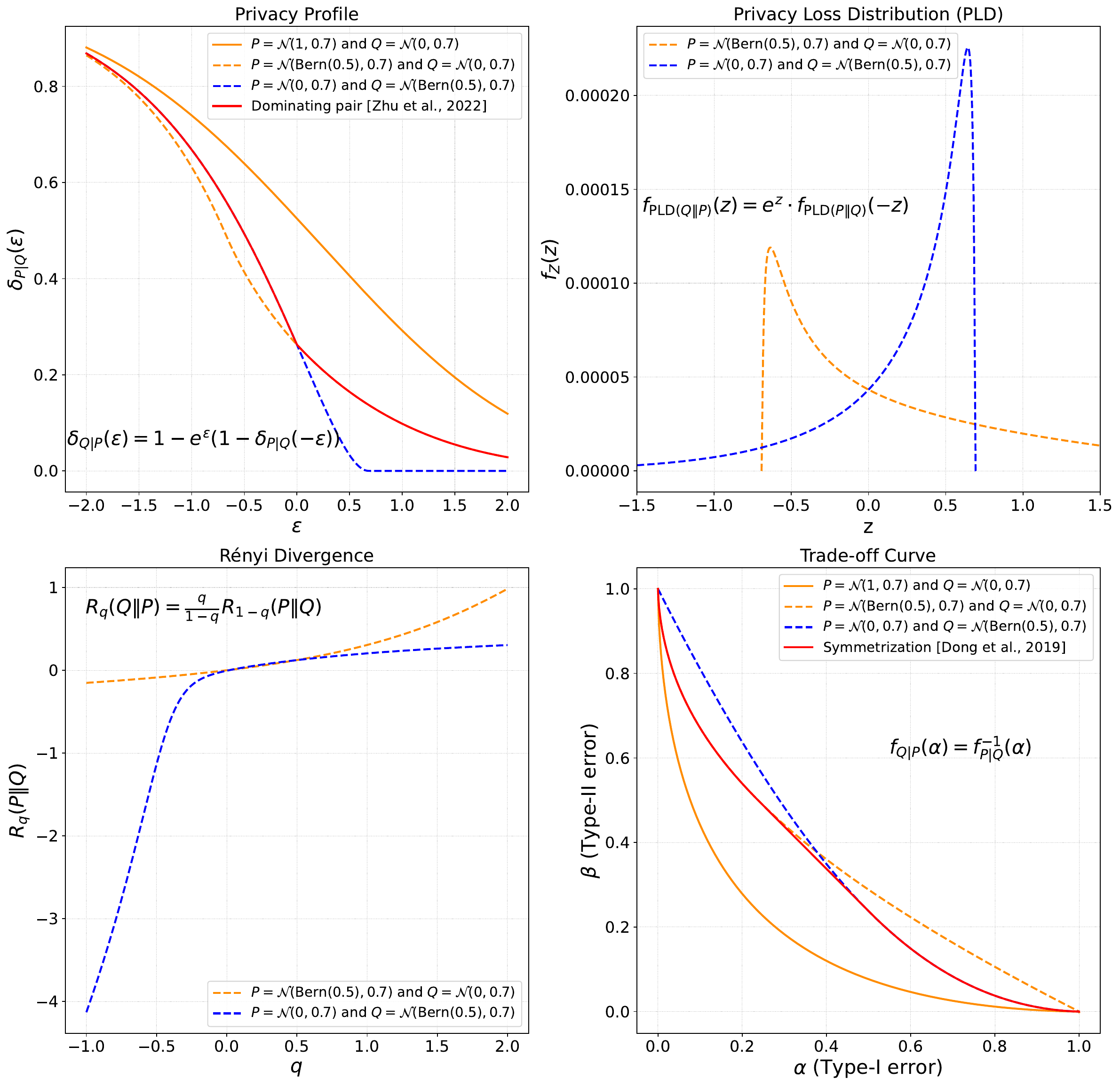}
	\caption{\label{fig:notion_comparisions} Visualization of the four functional notions of DP, namely privacy profile $\del_{P|Q}(\eps)$ as a function of $\eps$, the generalized density function of privacy loss distribution $\privloss{P}{Q}$, the R\'enyi divergence $\Ren{\q}{P}{Q}$ as a function of order $\q$, and the trade-off function $f_{P|Q}(\alpha)$ for hypothesis testing between $P$ and $Q$. We also provide the reversal theorems for each of the plots.}
\end{figure}

This privacy amplification result was first provided by~\citet{li2012sampling}, and since has appeared in several works~\citep{balle2020privacy,balle2018privacy,abadi2016deep,steinke2022composition}. But unlike other works, we present this amplification effect as a \emph{single curve} that exactly captures the impact of subsampling on both directions. This effect is visualized in the top-left plot in~\autoref{fig:notion_comparisions} (solid orange curve vs. dashed orange curve).

\textbf{Imprecise Handling of Asymmetry.} We observe that several works on privacy handle asymmetric notions of DP in a somewhat imprecise manner, which can introduce significant slack in the analysis or numerical bounds. For example, \citet[Theorem 4.2]{dong2019gaussian} uses the biconjugation operation \( \min\{f_{P|Q}, f_{Q|P}\}^{**} \) to quantify amplification for Poisson subsampling on the trade-off curve \( f_{P|Q} \) (corresponding to the subsampled profile \( \del_{P|Q} \) in~\autoref{thm:subsampling}), leading to an overestimation of the actual trade-off curve \( f_{P|Q} \) (see the bottom-right plot in \autoref{fig:notion_comparisions} for a comparison of the overestimated trade-off after symmetrization in the red line versus the actual trade-off \( f_{P|Q} \) in the orange dashed line). Similarly, \citet[Proposition 30]{dong2019gaussian} defines a dominating profile under subsampling by taking the max of \( \del_{P|Q} \) and \( \del_{Q|P} \) (see the top-left plot in \autoref{fig:notion_comparisions} to compare the overestimated privacy profile in red with the actual profiles).

The reasoning behind these operations is that a DP guarantee function should hold in both directions, requiring consideration of the worst-case aspects from each direction. However, these symmetrization steps can introduce small gaps that may compound significantly when several privacy profiles are composed. Moreover, these operations disrupt the equivalence across different privacy notions, as after symmetrization, converting to another notion, such as the privacy profile, no longer aligns with the actual privacy profile.

\textbf{Proposed Solution.} We note that all these functional notions of privacy possess a reversal property (see Remarks~\ref{rem:dp_reversal},~\ref{rem:rdp_reversal},~\ref{rem:f_reversal}, and~\ref{rem:pld_reversal}), which allows us to bypass the need for symmetrization operations. The key idea is to retain the chosen characterizing function \emph{in only one direction}, without attempting to control its asymmetry. Additionally, while composing functional notions, we refrain from changing the notion of adjacency (or its direction). Doing this enables lossless composition operations, thanks to~\autoref{thm:compose_privacy_profile} or R\'enyi DP composition~\citep{mironov2017renyi}, without running the risk of accidentally underestimating the privacy. When a pointwise DP guarantee is required, we can leverage the reversal properties to query the curve at a specific budget constraint in both directions, providing the max of the two, which results in an exactly-tight pointwise DP guarantee. Additionally, we also avoid the overhead of maintaining the functional representation, such as the PLD, in both directions as currently done in Google's \href{https://github.com/google/differential-privacy/blob/0a80e37c1d08feaae798f9ce54cf12c1781a5927/python/dp_accounting/dp_accounting/pld/privacy_loss_distribution.py#L99-L101}{PLDAccountant}.

\section{Conclusion}
\label{sec:conclusion}

In summary, this paper presents a novel interpretation of differential privacy by leveraging time-frequency dualities across various functional representations, including privacy profiles, R\'enyi divergence, and privacy loss distributions. By framing these within the context of Laplace transforms, we develop a versatile analytical toolkit for DP, enhancing both theoretical understanding and practical composition methods. Our approach addresses limitations in existing adaptive composition bounds, provides continuous guarantees for composed privacy profiles, and bridges gaps between different DP frameworks without needing approximations or symmetrizations. Together, these results push forward the capabilities of differential privacy research, setting a foundation for more nuanced and robust privacy-preserving algorithms.

\bibliography{references}

\begin{thebibliography}{32}
\providecommand{\natexlab}[1]{#1}
\providecommand{\url}[1]{\texttt{#1}}
\expandafter\ifx\csname urlstyle\endcsname\relax
  \providecommand{\doi}[1]{doi: #1}\else
  \providecommand{\doi}{doi: \begingroup \urlstyle{rm}\Url}\fi

\bibitem[Abadi et~al.(2016)Abadi, Chu, Goodfellow, McMahan, Mironov, Talwar,
  and Zhang]{abadi2016deep}
M.~Abadi, A.~Chu, I.~Goodfellow, H.~B. McMahan, I.~Mironov, K.~Talwar, and
  L.~Zhang.
\newblock Deep learning with differential privacy.
\newblock In \emph{Proceedings of the 2016 ACM SIGSAC conference on computer
  and communications security}, pages 308--318, 2016.

\bibitem[Asoodeh et~al.(2020)Asoodeh, Liao, Calmon, Kosut, and
  Sankar]{asoodeh2020better}
S.~Asoodeh, J.~Liao, F.~P. Calmon, O.~Kosut, and L.~Sankar.
\newblock A better bound gives a hundred rounds: Enhanced privacy guarantees
  via f-divergences.
\newblock In \emph{2020 IEEE International Symposium on Information Theory
  (ISIT)}, pages 920--925. IEEE, 2020.

\bibitem[Balle and Wang(2018)]{balle2018improving}
B.~Balle and Y.-X. Wang.
\newblock Improving the gaussian mechanism for differential privacy: Analytical
  calibration and optimal denoising.
\newblock In \emph{International Conference on Machine Learning}, pages
  394--403. PMLR, 2018.

\bibitem[Balle et~al.(2018)Balle, Barthe, and Gaboardi]{balle2018privacy}
B.~Balle, G.~Barthe, and M.~Gaboardi.
\newblock Privacy amplification by subsampling: Tight analyses via couplings
  and divergences.
\newblock \emph{Advances in neural information processing systems}, 31, 2018.

\bibitem[Balle et~al.(2020{\natexlab{a}})Balle, Barthe, and
  Gaboardi]{balle2020privacy}
B.~Balle, G.~Barthe, and M.~Gaboardi.
\newblock Privacy profiles and amplification by subsampling.
\newblock \emph{Journal of Privacy and Confidentiality}, 10\penalty0 (1),
  2020{\natexlab{a}}.

\bibitem[Balle et~al.(2020{\natexlab{b}})Balle, Barthe, Gaboardi, Hsu, and
  Sato]{balle2020hypothesis}
B.~Balle, G.~Barthe, M.~Gaboardi, J.~Hsu, and T.~Sato.
\newblock Hypothesis testing interpretations and renyi differential privacy.
\newblock In \emph{International Conference on Artificial Intelligence and
  Statistics}, pages 2496--2506. PMLR, 2020{\natexlab{b}}.

\bibitem[Bassily et~al.(2014)Bassily, Smith, and Thakurta]{bassily2014private}
R.~Bassily, A.~Smith, and A.~Thakurta.
\newblock Private empirical risk minimization: Efficient algorithms and tight
  error bounds.
\newblock In \emph{2014 IEEE 55th annual symposium on foundations of computer
  science}, pages 464--473. IEEE, 2014.

\bibitem[Bun and Steinke(2016)]{bun2016concentrated}
M.~Bun and T.~Steinke.
\newblock Concentrated differential privacy: Simplifications, extensions, and
  lower bounds.
\newblock In \emph{Theory of Cryptography Conference}, pages 635--658.
  Springer, 2016.

\bibitem[Canonne et~al.(2020)Canonne, Kamath, and Steinke]{canonne2020discrete}
C.~L. Canonne, G.~Kamath, and T.~Steinke.
\newblock The discrete gaussian for differential privacy.
\newblock \emph{Advances in Neural Information Processing Systems},
  33:\penalty0 15676--15688, 2020.

\bibitem[Cohen(2007)]{cohen2007numerical}
A.~M. Cohen.
\newblock \emph{Numerical methods for Laplace transform inversion}, volume~5.
\newblock Springer Science \& Business Media, 2007.

\bibitem[Dong et~al.(2019)Dong, Roth, and Su]{dong2019gaussian}
J.~Dong, A.~Roth, and W.~J. Su.
\newblock Gaussian differential privacy.
\newblock \emph{arXiv preprint arXiv:1905.02383}, 2019.

\bibitem[Doroshenko et~al.(2022)Doroshenko, Ghazi, Kamath, Kumar, and
  Manurangsi]{doroshenko2022connect}
V.~Doroshenko, B.~Ghazi, P.~Kamath, R.~Kumar, and P.~Manurangsi.
\newblock Connect the dots: Tighter discrete approximations of privacy loss
  distributions.
\newblock \emph{arXiv preprint arXiv:2207.04380}, 2022.

\bibitem[Dwork(2006)]{dwork2006differential}
C.~Dwork.
\newblock Differential privacy.
\newblock In \emph{International colloquium on automata, languages, and
  programming}, pages 1--12. Springer, 2006.

\bibitem[Dwork and Rothblum(2016)]{dwork2016concentrated}
C.~Dwork and G.~N. Rothblum.
\newblock Concentrated differential privacy.
\newblock \emph{arXiv preprint arXiv:1603.01887}, 2016.

\bibitem[Dwork et~al.(2010)Dwork, Rothblum, and Vadhan]{dwork2010boosting}
C.~Dwork, G.~N. Rothblum, and S.~Vadhan.
\newblock Boosting and differential privacy.
\newblock In \emph{2010 IEEE 51st annual symposium on foundations of computer
  science}, pages 51--60. IEEE, 2010.

\bibitem[Dwork et~al.(2014)Dwork, Roth, et~al.]{dwork2014algorithmic}
C.~Dwork, A.~Roth, et~al.
\newblock The algorithmic foundations of differential privacy.
\newblock \emph{Foundations and Trends{\textregistered} in Theoretical Computer
  Science}, 9\penalty0 (3--4):\penalty0 211--407, 2014.

\bibitem[Dyke and Dyke(2001)]{dyke2001introduction}
P.~P. Dyke and P.~Dyke.
\newblock \emph{An introduction to Laplace transforms and Fourier series}.
\newblock Springer, 2001.

\bibitem[Gopi et~al.(2021)Gopi, Lee, and Wutschitz]{gopi2021numerical}
S.~Gopi, Y.~T. Lee, and L.~Wutschitz.
\newblock Numerical composition of differential privacy.
\newblock \emph{Advances in Neural Information Processing Systems},
  34:\penalty0 11631--11642, 2021.

\bibitem[Kairouz et~al.(2015)Kairouz, Oh, and
  Viswanath]{kairouz2015composition}
P.~Kairouz, S.~Oh, and P.~Viswanath.
\newblock The composition theorem for differential privacy.
\newblock In \emph{International conference on machine learning}, pages
  1376--1385. PMLR, 2015.

\bibitem[Koskela et~al.(2020)Koskela, J{\"a}lk{\"o}, and
  Honkela]{koskela2020computing}
A.~Koskela, J.~J{\"a}lk{\"o}, and A.~Honkela.
\newblock Computing tight differential privacy guarantees using fft.
\newblock In \emph{International Conference on Artificial Intelligence and
  Statistics}, pages 2560--2569. PMLR, 2020.

\bibitem[Li et~al.(2012)Li, Qardaji, and Su]{li2012sampling}
N.~Li, W.~Qardaji, and D.~Su.
\newblock On sampling, anonymization, and differential privacy or,
  k-anonymization meets differential privacy.
\newblock In \emph{Proceedings of the 7th ACM Symposium on Information,
  Computer and Communications Security}, pages 32--33, 2012.

\bibitem[Miller(1951)]{miller1951moment}
J.~F. Miller.
\newblock Moment-generating functions and laplace transforms.
\newblock \emph{Journal of the Arkansas Academy of Science}, 4\penalty0
  (1):\penalty0 97--100, 1951.

\bibitem[Mironov(2017)]{mironov2017renyi}
I.~Mironov.
\newblock R{\'e}nyi differential privacy.
\newblock In \emph{2017 IEEE 30th computer security foundations symposium
  (CSF)}, pages 263--275. IEEE, 2017.

\bibitem[Murtagh and Vadhan(2015)]{murtagh2015complexity}
J.~Murtagh and S.~Vadhan.
\newblock The complexity of computing the optimal composition of differential
  privacy.
\newblock In \emph{Theory of Cryptography Conference}, pages 157--175.
  Springer, 2015.

\bibitem[Oppenhiem et~al.(1996)Oppenhiem, Willsky, and
  Nawab]{oppenhiem1996signals}
A.~V. Oppenhiem, A.~S. Willsky, and S.~H. Nawab.
\newblock Signals and systems, 1996.

\bibitem[Orloff(2015)]{orloff2015uniqueness}
J.~Orloff.
\newblock Uniqueness of laplace transform, 2015.

\bibitem[Privacy()]{differentialprivacyTightZCDP}
D.~Privacy.
\newblock {T}ight {R}{D}{P} \& z{C}{D}{P} {B}ounds from {P}ure {D}{P} ---
  differentialprivacy.org.
\newblock \url{https://differentialprivacy.org/pdp-to-zcdp/}.
\newblock [Accessed 12-11-2024].

\bibitem[Sason and Verd{\'u}(2016)]{sason2016f}
I.~Sason and S.~Verd{\'u}.
\newblock $ f $-divergence inequalities.
\newblock \emph{IEEE Transactions on Information Theory}, 62\penalty0
  (11):\penalty0 5973--6006, 2016.

\bibitem[Sommer et~al.(2018)Sommer, Meiser, and Mohammadi]{sommer2018privacy}
D.~Sommer, S.~Meiser, and E.~Mohammadi.
\newblock Privacy loss classes: The central limit theorem in differential
  privacy.
\newblock \emph{Cryptology ePrint Archive}, 2018.

\bibitem[Steinke(2022)]{steinke2022composition}
T.~Steinke.
\newblock Composition of differential privacy \& privacy amplification by
  subsampling.
\newblock \emph{arXiv preprint arXiv:2210.00597}, 2022.

\bibitem[Wang et~al.(2015)Wang, Fienberg, and Smola]{wang2015privacy}
Y.-X. Wang, S.~Fienberg, and A.~Smola.
\newblock Privacy for free: Posterior sampling and stochastic gradient monte
  carlo.
\newblock In \emph{International Conference on Machine Learning}, pages
  2493--2502. PMLR, 2015.

\bibitem[Zhu et~al.(2022)Zhu, Dong, and Wang]{zhu2022optimal}
Y.~Zhu, J.~Dong, and Y.-X. Wang.
\newblock Optimal accounting of differential privacy via characteristic
  function.
\newblock In \emph{International Conference on Artificial Intelligence and
  Statistics}, pages 4782--4817. PMLR, 2022.

\end{thebibliography}

\newpage

\appendix

\section{Appendix}
\subsection{Table of Properties of Laplace Transform}
\label{app:prop_table}
\begin{table}[!htbp]
	\centering
	\caption{\label{tab:laplaceTransformProp}Properties of the Laplace Transform. Let $g(t)$ and $h(t)$ be two functions defined for $t \in \R$ and let $a, b \in \R$ be arbitrary constants.}
	\begin{tabularx}{\textwidth}{@{}rML@{}}
		\toprule
		\textbf{Property} & \textbf{Expression} & \multicolumn{1}{l}{} \\
		\midrule
		Linearity : &
			\begin{aligned}
				\UniLap{a g(t) + b h(t)}(s) &= a \UniLap{g(t)}(s) + b \UniLap{h(t)}(s) \\ 
				\BiLap{a g(t) + b h(t)}(s) &= a \BiLap{g(t)}(s) + b \BiLap{h(t)}(s) 
			\end{aligned} & eq:Linearity \\
		\hline 
		Time-Shifting : &
			\begin{aligned}
				\UniLap{g(t - a)\indic{t > a}}(s) &= e^{-as} \UniLap{g(t)}(s), \text{ for } a > 0 \\
				\BiLap{g(t - a)}(s) &= e^{-as} \BiLap{g(t)}(s), \text{ for } a \in \R 
			\end{aligned} & eq:TimeShifting \\
		\hline
		Frequency-Shifting : & 
			\begin{aligned}
				\UniLap{e^{at} g(t)}(s) &= \UniLap{g(t)}(s - a) \\ 
				\BiLap{e^{at} g(t)}(s) &= \BiLap{g(t)}(s - a)
			\end{aligned} & eq:FrequencyShifting \\
		\hline
		Time-Scaling : & 
			\begin{aligned}
				\UniLap{g(at)}(s) &= \frac{1}{a} \UniLap{g(t)}\left(\frac{s}{a}\right) \text{ for } a > 0 \\
				\BiLap{g(at)}(s) &= \frac{1}{|a|} \BiLap{g(t)}\left(\frac{s}{a}\right) \text{ for } a \in \R
			\end{aligned} & eq:Scaling \\
		\hline
		Reversal : &
			\BiLap{g(-t)}(s) = \BiLap{g(t)}(-s) & eq:Reversal \\
		\hline
		Derivative : &
			\begin{aligned}
				\UniLap{\dot g(t)}(s) &= s\UniLap{g(t)}(s) - g(0^+) \\
				\BiLap{\dot g(t)}(s) &= s\BiLap{g(t)}(s) 
			\end{aligned} & eq:Differentiation \\
		\hline
		Integration : & 
			\begin{aligned}
				\UniLap{\int_{0}^t g(t)dt}(s) &= \frac{1}{s}\UniLap{g(t)}(s), \text{ for } \mathfrak{Re}(s) > 0 \\
				\BiLap{\int_{-\infty}^t g(t)dt}(s) &= \frac{1}{s}\BiLap{g(t)}(s), \text{ for } \mathfrak{Re}(s) > 0
			\end{aligned} & eq:BilatIntegration \\
		\hline
		Convolution : & 
			\begin{aligned} 
				\UniLap{\int_0^t g(\tau)h(t-\tau) d\tau}(s) &= \UniLap{g(t)}(s)\cdot\UniLap{h(t)}(s) \\
				\BiLap{\int_{-\infty}^\infty g(\tau)h(t-\tau)d\tau}(s) &= \BiLap{g(t)}(s) \cdot \BiLap{h(t)}(s) 
			\end{aligned} & eq:Convolution \\
		\bottomrule
	\end{tabularx}
\end{table}

\newpage
\subsection{Deferred Proofs for Section~\ref{sec:characterization}}
\label{app:characterization}

\begin{reptheorem}{thm:dp_as_laplace}
	For a random variable $X$, let $F_X(t) \eqdef \Pr[X \leq t]$ denote its \emph{cumulative distribution function} and $f_X(t)$ denote its \emph{generalized probability density function}. 
	Let $P$ and $Q$ be probability distributions and $\Z \sim \privloss{P}{Q}$ and $\Z' \sim \privloss{Q}{P}$ denote their privacy loss random variables.
	If $\Z \sim \privloss{P}{Q}$ and $\Z' \sim \privloss{Q}{P}$, then for all $\eps \in \R$, 
	\begin{align}
		\del_{P|Q}(\eps) &= \UniLap{1 - F_\Z(t + \eps)}(1) \label{eqn:dp_as_laplace_1_appendix}\\
				 &= e^\eps \cdot \UniLap{F_{\Z'}(-t-\eps)} (-1) \label{eqn:dp_as_laplace_2_appendix}\\
				 &= e^\eps \cdot \UniLap{f_{\Z'}(-t-\eps)}(-1) - \UniLap{f_\Z(t + \eps)}(1) \label{eqn:dp_as_laplace_3_appendix} \\
				 &= \UniLap{f_\Z(t + \eps)}(0)  - e^\eps \cdot \UniLap{f_{\Z'}(-t-\eps)}(0). \label{eqn:dp_as_laplace_4_appendix}
	\end{align}
	And, for all $\q \in \mathrm{ROC}_\mathcal{B}\{f_{\Z'}\}$ (or equivalently, $1-\q \in \mathrm{ROC}_{\mathcal{B}}\{f_{\Z}\}$),
	\begin{align}
		e^{(\q - 1) \cdot \Ren{\q}{P}{Q}} = \Eren{\q}{P}{Q} = \mathcal{B}\{f_{\Z}(t)\}(1-\q) = \mathcal{B}\{f_{\Z'}(t)\}(\q). \label{eqn:rdp_as_laplace_appendix}
	\end{align}
\end{reptheorem}
\begin{proof}
	Denote the set where privacy loss $\priv{P}{Q}(\thet)$ exceeds $\eps$ as
	\begin{align}
		S^*_{>\eps} = \{ \thet \in \OO : P(\thet) > e^\eps Q(\thet)\}.
	\end{align}
	Note that for all $S \subset \OO$, and all $\eps \in \R$, we have 
	\begin{equation}
		P(S) - e^\eps Q(S) \leq P(S^*_{>\eps}) - e^\eps Q(S^*_{>\eps}) = \del_{P|Q}(\eps),
	\end{equation}
	because $S^*_{>\eps}$ includes any and all points where $P(\thet) > e^\eps Q(\thet)$. We can express the probabilities $P(S^*_{>\eps})$ and $Q(S^*_{>\eps})$ using the Laplace transform as follows:

	\noindent\begin{minipage}{.5\linewidth}
		\begin{align}
			P(S^*_{>\eps}) &= \int_{S^*_{>\eps}} P(\thet) d\thet \\
				       &= \int_{S^*_{>\eps}} \left(\frac{P(\thet)}{Q(\thet)}\right) Q(\thet) d\thet \\
				       &= \int_{S^*_{>\eps}} e^{-\priv{Q}{P}(\thet)} Q(\thet) d\thet \\
				       &= \int_{\eps^+}^\infty e^{t} \int_{\{\thet \in \OO : \priv{Q}{P}(\thet) = -t\}} Q(\thet) d\thet \\
				       &= \int_{\eps^+}^\infty e^{t} \int_{-t^-}^{-t^+} F_{\Z'}(u)du \\
				       &= \int_{\eps^+}^\infty e^{t} f_{\Z'}(-t) dt \\
				       &= e^{\eps} \int_{0^+}^\infty e^{t'} f_{\Z'}(-t' - \eps) dt' \\
				       &= e^{\eps} \UniLap{f_{\Zc}(-t-\eps)}(-1).
		\end{align}
	\end{minipage}
	\begin{minipage}{.5\linewidth}
		\begin{align}
			Q(S^*_{>\eps}) &= \int_{S^*_{>\eps}} Q(\thet) d\thet \\
				       &= \int_{S^*_{>\eps}} \left(\frac{Q(\thet)}{P(\thet)}\right) P(\thet) d\thet \\
				       &= \int_{S^*_{>\eps}} e^{-\priv{P}{Q}(\thet)} P(\thet) d\thet \\
				       &= \int_{\eps^+}^\infty e^{-t} \int_{\{\thet \in \OO : \priv{P}{Q}(\thet) = t\}} P(\thet) d\thet \\
				       &= \int_{\eps^+}^\infty e^{-t} \int_{t^-}^{t^+} F_{\Z}(u)du \\
				       &= \int_{\eps^+}^\infty e^{-t} f_\Z(t) dt \\
				       &= e^{-\eps} \int_{0^+}^\infty e^{-t'} f_\Z(t' + \eps) dt \\
				       &= e^{-\eps} \UniLap{f_{\Z}(t + \eps)}(1).
		\end{align}
	\end{minipage}
	Combining the two, we get equation~\eqref{eqn:dp_as_laplace_3_appendix}:
	\begin{equation}
		\del_{P|Q}(\eps) = e^{\eps} \cdot \UniLap{f_{\Z'}(-t-\eps)}(-1) - \UniLap{f_{\Z}(t+\eps)}(1).
	\end{equation}
	Alternatively, we can express the profile directly as:
	\begin{align}
		P(S^*_{>\eps}) - e^\eps Q(S^*_{>\eps}) &= \int_{S^*_{>\eps}} (1 - e^\eps \frac{Q(\thet)}{P(\thet)}) P(\thet) d\thet \\
							     &= \int_{S^*_{>\eps}} (1 - e^{\eps - \priv{P}{Q}(\thet)}) P(\thet) d\thet \\
							     &= \int_{\eps^+}^\infty (1-e^{\eps - t}) f_\Z(t) dt \\
							     &= \int_{\eps^+}^\infty f_\Z(t)dt - \int_{0^+}^\infty e^{-t'} f_\Z(t' + \eps) dt' \tag{Change $t = t'+\eps$}\\
							     &= 1 - F_\Z(\eps^+) - \UniLap{f_{\Z}(t+\eps)}(1) \\
							     &\overset{\eqref{eq:Differentiation}}{=} \UniLap{1 - F_{\Z}(t+\eps)}(1).
	\end{align}
	Similarly, we can express it in terms of $\Z'$ as
	\begin{align}
		P(S^*_{>\eps}) - e^\eps Q(S^*_{>\eps}) &= \int_{S^*_{>\eps}} (\frac{P(\thet)}{Q(\thet)} - e^\eps) Q(\thet) d\thet \\
							 &= \int_{S^*_{>\eps}} (e^{\priv{P}{Q}(\thet)} - e^{\eps}) Q(\thet) d\thet \\
							 &= \int_{\eps^+}^{\infty} (e^{t}-e^{\eps}) f_\Zc(-t) dt \\
							 &= e^\eps\left(\int_{0^+}^\infty e^{t'} f_\Zc(-t' - \eps) dt'- \int_{\eps^+}^{\infty} f_\Zc(-t)dt \right) \tag{Change $t = t' +\eps$} \\
							 &= e^\eps \left(\UniLap{f_{\Z'}(-t-\eps)}(-1) -  F_{\Z'}(-\eps^-)\right)  \\
							 &\overset{\eqref{eq:Differentiation}}{=} e^\eps \UniLap{F_{\Z'}(-t-\eps)}(-1).
	\end{align}
	For showing~\eqref{eqn:dp_as_laplace_4_appendix}, we apply the derivative property of Laplace transform to~\eqref{eqn:dp_as_laplace_1_appendix} and~\eqref{eqn:dp_as_laplace_2_appendix} to get
 	\begin{align}
		\del_{P|Q}(\eps) &= \UniLap{1-F_\Z(t+\eps)}(1) \overset{\eqref{eq:Differentiation}}{=} - \UniLap{f_\Z(t+\eps)}(1) + 1 - F_\Z(\eps^+), \quad \text{and}  \\
 		\del_{P|Q}(\eps) &= e^\eps \cdot \UniLap{F_{\Z'}(-t - \eps)}(-1) \overset{\eqref{eq:Differentiation}}{=} e^\eps \cdot \left(\UniLap{f_{\Z'}(-t-\eps)}(-1)  - F_{\Z'}(-\eps^-)\right).
 	\end{align}
 	Adding the above two equations and subtracting~\eqref{eqn:dp_as_laplace_3_appendix} from it, we get
 	\begin{align}
		\del_{P|Q}(\eps) &= 1-F_{\Z}(\eps^+) - e^\eps F_\Zc(-\eps^-) \\
				 &= \Pr[\Z > \eps] - e^\eps \Pr[\Zc < -\eps] \\
				 &= \int_{0^+}^{\infty} e^{0 \cdot t} \cdot f_\Z(t + \eps)dt - e^\eps \cdot \int_{0^+}^{\infty} e^{0\cdot t} f_{\Z'}(-t-\eps) dt \\
				 &= \UniLap{f_\Z(t + \eps)}(0) - e^\eps \cdot \UniLap{f_{\Z'}(-t-\eps)}(0).
 	\end{align}
	For the last part, recall from definition that R\'enyi divergence $\Ren{\q}{P}{Q} = \frac{1}{\q-1} \log \Eren{\q}{P}{Q}$, for which we show the following two equivalences:
	
	\noindent\begin{minipage}{.5\linewidth}
		\begin{align}
			\Eren{\q}{P}{Q}		&= \int_\OO \left(\frac{P(\thet)}{Q(\thet)}\right)^{\q-1}P(\thet)d\thet \\
						&= \int_\OO e^{(\q-1)\priv{P}{Q}(\thet)}P(\thet)d\thet \\
						&= \int_{-\infty}^\infty e^{(\q-1)t} f_\Z(t)dt \\
						&= \BiLap{f_\Z(t)}(1-\q).
		\end{align}
	\end{minipage}
	\begin{minipage}{.5\linewidth}
		\begin{align}
			\Eren{\q}{P}{Q}		 &= \int_\OO \left(\frac{P(\thet)}{Q(\thet)}\right)^{\q}Q(\thet)d\thet \\
						 &= \int_\OO e^{-\q \priv{Q}{P}(\thet)}Q(\thet)d\thet \\
						 &= \int_{-\infty}^\infty e^{-\q t} f_\Zc(t)dt \\
						 &= \BiLap{f_{\Z'}(t)}(\q).
		\end{align}
	\end{minipage}
\end{proof}

\begin{reptheorem}{thm:rr_privacy_profile}[Privacy profile of randomized response]
	Fix $\eps > 0$ and $\del \in [0, 1]$.
	Let $\M_\mathrm{RR} : \{0, 1\} \rightarrow \{0, 1\} \times \{\bot, \top\} $ be the randomized response mechanism, which has the following output probabilities.
	\begin{equation}
		\label{eqn:rr_pld_app}
		\M_\mathrm{RR}(0) = \begin{cases}
			(0, \bot) & \text{with probability}\ \del,\\ (0, \top) & \text{with probability}\ \frac{(1-\del)e^{\eps}}{e^\eps+1},\\
			(1, \top) & \text{with probability}\ \frac{(1-\del)}{e^\eps+1}, \\
			(1, \bot) & \text{with probability}\ 0,
		\end{cases}
		\M_\mathrm{RR}(1) = \begin{cases}
			(0, \bot) & \text{with probability}\ 0,\\ 
			(0, \top) & \text{with probability}\ \frac{(1-\del)}{e^\eps+1},\\
			(1, \top) & \text{with probability}\ \frac{(1-\del)e^{\eps}}{e^\eps+1}, \\
			(1, \bot) & \text{with probability}\ \del.
		\end{cases}
	\end{equation}
	For $P = \M_\mathrm{RR}(0)$ and $Q = \M_\mathrm{RR}(1)$, the privacy profiles are
	\begin{equation}
		\forall t \in \R \ : \ \del_{P|Q}(t) = \del_{Q|P}(t) =
		\begin{cases}
			\del & \text{if}\ \eps < t, \\
			1 - \frac{(e^t+1)}{e^\eps + 1}(1-\del) &\text{if} \ -\eps < t \leq \eps, \\ 
			1-e^t(1 - \del) & \text{if} \ t \leq -\eps.
		\end{cases}
	\end{equation}
\end{reptheorem}
\begin{proof}
	Let $S_1 = \{(0, \bot)\}$, $S_2 = S_1 \cup \{(1, \bot)\}$, and $S_3 = S_2 \cup \{(1, \top)\}$.
	From~\eqref{eqn:dp_pld}, 
	\begin{align}
		\del_{P|Q}(t) &= \Pr_{\Z \leftarrow \privloss{P}{Q}}[\Z > t] - e^\eps \cdot \Pr_{\Z' \leftarrow \privloss{Q}{P}}[\Z' < -t] \\
			      &= \Pr_{P}[\log \frac{P(\Thet)}{Q(\Thet)} > t] - e^\eps\cdot \Pr_{Q}[\log \frac{Q(\Thet)}{P(\Thet)} < -t] \\
			      &= \Pr_{P}[P(\Thet) > e^t \cdot Q(\Thet)] - e^\eps\cdot \Pr_{Q}[P(\Thet) < e^{t} \cdot Q(\Thet)] \\
			      &= \begin{cases}
				      P(S_1) - e^t \cdot Q(S_1)& \text{if} \ \eps < t, \\
				      P(S_2) - e^t \cdot Q(S_2) & \text{if} \ -\eps < t \leq \eps, \\
				      P(S_3) - e^t \cdot Q(S_3) & \text{otherwise}
			      \end{cases} \\
			      &= \begin{cases}
				      \del& \text{if} \ \eps < t, \\
				      \del + \frac{1 - \del}{e^\eps + 1} \cdot (e^\eps - e^t) & \text{if} \ -\eps < t \leq \eps, \\
				      1 - e^t \cdot (1 - \del) & \text{otherwise}
			      \end{cases}
	\end{align}
	The same holds for $\del_{Q|P}$ due to symmetry of~\eqref{eqn:rr_pld_app}.
\end{proof}
\begin{reptheorem}{thm:rdp_of_rr}[R\'enyi DP of $(\eps, 0)$-Randomized Response]
	For any $\eps > 0$ and $\del = 0$, the output distributions of randomized response mechanism in~\autoref{thm:rr_privacy_profile} exhibit a R\'enyi divergence
	\begin{equation}
		\label{eqn:rdp_of_rr_app}
		\forall \q \in \Com \ \text{s.t.} \ \mathfrak{Re}(\q) \not\in\{0, 1\} : \Ren{\q}{P}{Q} = \frac{1}{\q-1} \log \left(\frac{e^\eps}{1+e^{\eps}} e^{-\q\eps} + \frac{1}{1+e^{\eps}} e^{\q\eps}\right).
	\end{equation}
\end{reptheorem}
\begin{proof}
	From~\autoref{thm:rr_privacy_profile}, when $\del = 0$, the privacy profile of randomized response algorithm's output-distributions $P$ and $Q$ is 
	\begin{equation}
		\del_{P|Q}(t) = \begin{cases}
			0 &\text{if}\  \eps < t, \\
			\frac{e^{\eps} - e^t}{1+e^{\eps}} &\text{if} \ -\eps < t \leq \eps, \\
			1 - e^t & \text{otherwise}.
		\end{cases}
	\end{equation}
	From the equivalence~\eqref{eqn:dp_profile_to_rdp} of~\autoref{thm:dp_profile_to_rdp}, 
	\begin{align}
		\frac{e^{(\q-1)R_\q(P\Vert Q)}}{\q(\q-1)} &= \BiLap{\del_{P|Q}(t)}(1-\q)\\
		&=\int_{-\infty}^\infty e^{(\q-1)t}\del_{P|Q}(t)\mathrm{d}t \\
		&= \int_{-\infty}^{-\eps} e^{(\q-1)t} \cdot (1-e^t)\mathrm{d}t + \int_{-\eps}^{\eps}e^{(\q-1)t} \cdot\frac{e^{\eps} - e^t}{1+e^{\eps}}\mathrm{d}t \\
		&=\left[\frac{e^{(\q - 1) t}}{\q-1} - \frac{e^{\q t}}{\q}\right]_{-\infty}^{-\eps} + \frac{1}{1+e^{\eps}}\left[\frac{e^{\eps} \cdot e^{(\q-1)t}}{\q-1} - \frac{e^{\q t}}{\q}\right]_{-\eps}^{\eps} \\
		&= \left(\frac{e^{-(\q-1)\eps}}{\q-1} - \frac{e^{-\q\eps}}{\q}\right) + \frac{1}{1+e^{\eps}}\left[\left(\frac{e^{\q\eps}}{\q-1} - \frac{e^{\q \eps}}{\q}\right) - \left(\frac{e^{\eps} \cdot e^{-(\q-1)\eps}}{\q-1} - \frac{e^{-\q\eps}}{\q}\right)\right] \\
		&= \frac{e^{-(\q-1)\eps}}{\q-1}\cdot\left(1 - \frac{e^{\eps}}{1+e^{\eps}}\right) -\frac{e^{-\q\eps}}{\q}\cdot\left(1-\frac{1}{1+e^{\eps}}\right) + \frac{e^{\q\eps}}{1+e^{\eps}} \cdot \left(\frac{1}{\q-1} - \frac{1}{\q}\right) \\
		&= \frac{e^{-(\q-1)\eps}}{1+e^{\eps}} \cdot \left(\frac{1}{\q-1}-\frac{1}{\q}\right)+\frac{e^{\q\eps}}{1+e^{\eps}} \cdot \left(\frac{1}{\q-1} - \frac{1}{\q}\right) \\
		&=\frac{e^{\eps} \cdot e^{-\q\eps} + e^{\q\eps}}{1+e^{\eps}} \cdot \frac{1}{\q(\q-1)}.
	\end{align}
	Therefore, for any $\q \in \R \setminus \{0,1\}$ we can cancel $\q(\q-1)$ from the denominator in both sides, which proves the theorem statement for real orders. From dominated convergence theorem the theorem statement holds for complex orders as well on corresponding real orders (cf.~\autoref{ssec:laplace_transform}).
\end{proof}

\subsection{Deferred Proofs for Section~\ref{sec:composition}}
\label{app:composition}

\begin{lemma}[{\citep[Corollary 24]{steinke2022composition}}]
	\label{corr:steinke_approxdp_decomp}
	Let $P$ and $Q$ be probability distributions over $\OO$. Fix $\eps \geq 0$ and $\del \in [0, 1]$. Suppose $P, Q$ satisfy $(\epsilon, \delta)$-differential privacy. Then there exists distributions $A, B, P', Q'$ over $\OO$ such that
	\begin{align}
		P &= (1-\del)\frac{e^{\eps}}{e^{\eps} + 1}A + (1-\del) \frac{1}{e^{\eps} + 1}B + \del P', \\
		Q &= (1-\del)\frac{e^{\eps}}{e^{\eps} + 1}B + (1-\del) \frac{1}{e^{\eps} + 1}A + \del Q'.
	\end{align}
\end{lemma}

\begin{reptheorem}{thm:rr_dominates_approxdp}[Dominating Privacy Profile under $(\eps, \del)$-DP]
	Fix $\eps \geq 0$ and $\del \in [0, 1]$. Suppose distributions $P$ and $Q$ over $\OO$ satisfy $(\eps, \del)$-differential privacy. Then,
	\begin{equation}
		\label{eqn:rr_dominates_approxdp_app}
		\forall t \in \R \ : \ \del_{P|Q}(t) \leq \del_\mathrm{RR}(t) \quad \text{and}\quad \del_{Q|P}(t) \leq \del_\mathrm{RR}(t),
	\end{equation}
	where $\del_\mathrm{RR}(t)$ is the privacy profile of the randomized response mechanism $\M_\mathrm{RR}^{\eps, \del}$.
\end{reptheorem}
\begin{proof}
	Since the output distributions $P$ and $Q$ are $(\eps, \del)$-differentially private, Lemma~\ref{corr:steinke_approxdp_decomp} from~\citet{steinke2022composition} tells us that we can simulate these two distributions as post-processing of the randomized response mechanism $\M_\mathrm{RR}^{\eps,\del}$. To see this, imagine that $P = \M(\D)$ and $Q = \M(\D')$ are the output distributions of some mechanism $\M$. Define another mechanism $\mathcal{G}: \{0, 1\} \times \{\bot, \top\} \rightarrow \OO$ with output distribution:
	\begin{equation}
		\mathcal{G}(u) = 
		\begin{cases} 
			P' & \text{if} \ u = (0, \bot), \\
			A & \text{if} \ u = (0, \top), \\
			B & \text{if} \ u = (1, \top), \\
			Q' & \text{if} \ u = (1, \bot).
		\end{cases}
	\end{equation}
	From Lemma~\ref{corr:steinke_approxdp_decomp}, see that $\mathcal{G}(\M_\mathrm{RR}^{\eps,\del}(0)) = \M(\D)$ and $\mathcal{G}(\M_\mathrm{RR}^{\eps, \del}(1)) = \M(D')$. Therefore, by expressing $P$ and $Q$ in terms of distributions $P', A, B, Q'$, we can conclude that for all $t \in \mathbb{R}$,
	\begin{align}
		\del_{P|Q}(t) &= \sup_{S\subset\OO} P(S) - e^t\cdot Q(S) \\
			      &= \sup_{S\subset\OO} \left(\del \cdot P'(S) + \frac{(1-\del)(e^\eps - e^t)}{e^\eps+1} \cdot A(S) + \frac{(1-\del)(1 - e^{\eps +t})}{e^\eps + 1}\cdot B(S) - \del e^t \cdot Q'(S)\right) \\
			      &\leq \del + \begin{cases}
					0						& \text{if} \ \eps < t, \\ 
					\frac{(1-\del)(e^\eps - e^t)}{e^\eps + 1}	& \text{if} \ -\eps < t \leq \eps, \\
					(1-\del)(1 - e^t)				& \text{if} \ t \leq -\epsilon,
				\end{cases} \\
				&=\begin{cases}
					\del						& \text{if} \ \eps < t, \\ 
					1 - \frac{(e^t+1)(1-\del)}{e^\eps + 1}		& \text{if} \ -\eps < t \leq \eps, \\
					1 - e^t(1-\del)					& \text{if} \ t \leq -\eps.
				\end{cases}
	\end{align}
	Note that the expression on the right is the privacy profile of the randomized response mechanism $\M_\mathrm{RR}^{\eps,\del}$.
	An identical bound follows for $\del_{Q|P}(t)$ as well, with the switched roles: $P'\leftrightarrow Q'$ and $A \leftrightarrow B$.
\end{proof}

\begin{reptheorem}{thm:approxdp_composition_hetro}[{Tight Composition for $(\eps, \del)$-DP}]
	For any $\eps_i \geq 0$, $\del_i \in [0,1]$ for $i \in \{1, \cdots, k\}$, the $k$-fold composition of $(\eps_i, \del_i)$-differentially private mechanisms satisfies $(\eps, \del^{\otimes k}(\eps))$-DP for all $\eps$, defined recursively as
	\begin{equation}
		\label{eqn:approxdp_composition_hetro_app}
		\forall t \in \R \ : \ \del^{\otimes l}(t) = \del_l + \frac{(1-\del_l)}{e^{\eps_l} + 1}\left[ e^{\eps_l} \cdot \del^{\otimes l-1}(t -\eps_l) + \del^{\otimes l-1}(t +\eps_l)\right],
	\end{equation}
	with $\del^{\otimes 0}(t) = [1 - e^t]_+$.
\end{reptheorem}
\begin{proof}
	Let $P_{1:k}, Q_{1:k}$ be the joint output distributions of the $k$-fold composed mechanism on neighboring inputs. To prove the statement, we need to show that
	\begin{equation}
		\forall t \in \R \ : \ \del_{P_{1:k}|Q_{1:k}}(t) \leq \del^{\otimes}(t).
	\end{equation}
	Let's define $P_i^{x_{<i}}, Q_i^{x_{<i}}$ be the output distributions of the $i^{th}$ mechanism, conditioned on the preceding $i-1$ mechanisms' output being $x_{<i}$. From~\autoref{thm:rr_dominates_approxdp} and from~\autoref{lem:adaptive_profile}, we know that under adaptive $(\eps_i, \del_i)$-DP, the privacy profiles for conditional distributions are dominated as follows.
	\begin{equation}
		\label{eqn:A}
		\forall i \in \{1, \cdots, k\} \ : \ \sup_{x_{<i} \in \OO_{1}\times \cdots \times \OO_{i-1}} \del_{P^{x_{<i}}_i|Q^{x_{<i}}_i} (t) \leq \del_{\mathrm{RR}}^{\eps_i, \del_i}(t),
	\end{equation}
	where $\del_{\mathrm{RR}}^{\eps_i, \del_i}(t)$ is the privacy profile of $\M_\mathbb{RR}^{\eps_i, \del_i}$.

	Using this, we prove the theorem statement inductively.

	\textbf{Base step.} Let's denote the Heaveside step function as $H(t) := \mathbb{I}\{t > 0\}$. Then, we can write $\del^{\otimes 0}(t) = (1 - H(t)) \cdot (1 - e^t)$. Using this, we can express
	\begin{align}
		\del^{\otimes 1}(t) &= \del_1 + \frac{1-\del_1}{e^{\eps_1}+1} \left[e^{\eps_1} \cdot \del^{\otimes 0}(t -\eps_1) + \delta^{\otimes 0}(t +\eps_1)\right] \\
		&=\del_1 +  \frac{1-\del_1}{e^{\eps_1}+1} \left[e^\eps_1 \cdot (1 - H(t-\eps_1)) \cdot (1 - e^{t - \eps_1}) + (1 - H(t+\eps_1))\cdot (1 - e^{t+\eps_1})\right] \\
		&= 1+e^x(1-\del_1) +H(t-\eps_1) \cdot \frac{(1-\del_1)(e^t - e^{\eps_1})}{e^{\eps_1}+1}+ H(t+\eps_1) \cdot \frac{(1-\del_1)(e^{t+\eps_1} - 1)}{e^{\eps_1}+1} \label{eqn:B}\\
		&=\begin{cases}
			\del_1 & \text{if}\ \eps_1 < t, \\ 
			1 - \frac{(e^t+1)(1-\del_1)}{e^\eps_1 + 1} &\text{if} \ -\eps_1 < t \leq \eps_1, \\ 
			1 - e^t(1-\del_1) & \text{if} \ t \leq -\eps_1.
		\end{cases} \\
		&=\del_\mathrm{RR}^{\eps_1, \del_1}(x)
	\end{align}
	From~\eqref{eqn:A}, we therefore get that $\del_{P_1|Q_1}(t) \leq \del^{\otimes 1}(t)$ for all $t \in \R$.

	\textbf{Induction step.} Suppose for any $l \in \{2,\cdots, k\}$ the composition of first $l-1$ mechanisms have a privacy profile dominated by $\del^{\otimes l-1}$. More precisely,
	\begin{equation}
		\forall t \in \R \ : \ \del_{P_{1:l-1}|Q_{1:l-1}}(t) \leq \del^{\otimes l-1}(t).
	\end{equation}
	We need to show that 
	\begin{equation}
		\forall t \in \R \ : \ \del_{P_{1:l}|Q_{1:l}}(t) \leq \del^{\otimes l}(t).
	\end{equation}
	Recall that from the adaptive $(\eps_l, \del_l)$-DP assumption on the $l^{th}$ mechanism, \eqref{eqn:A} says that
	\begin{equation}
		\sup_{x_{<l} \in \OO_{1}\times \cdots \times \OO_{l-1}} \del_{P^{x_{<l}}_l|Q^{x_{<l}}_l} (t) \leq \del_\mathrm{RR}^{\eps_l, \del_l}(t).
	\end{equation}
	Therefore, from~\autoref{thm:compose_privacy_profile},~\autoref{lem:adaptive_profile} and the induction assumption, we have that
	\begin{align}
		\del_{P_{1:l}|Q_{1:l}}(t) &\leq \left(\del_{P_{1:l-1}|Q_{1:l-1}} \circledast \left(\ddot \del^{\eps_l, \del_l}_\mathrm{RR} - \dot \del^{\eps_l, \del_l}_\mathrm{RR} \right)\right) (t) \\
		&\leq \left(\del^{\otimes l-1} \circledast \left(\ddot \del^{\eps_l, \del_l}_\mathrm{RR} - \dot \del^{\eps_l, \del_l}_\mathrm{RR} \right)\right) (t) \\
		&= \int_{-\infty}^\infty \del^{\otimes l-1}(t-\tau) \times \left(\ddot \del^{\eps_l, \del_l}_\mathrm{RR}(\tau) - \dot \del^{\eps_l, \del_l}_\mathrm{RR}(\tau)\right) \mathrm{d}\tau.
	\end{align}
	To differentiate properly, in a manner that handles discontinuity, we state the function $\del_{\mathrm{RR}}^{\eps,\del}(t)$ in terms of Heaviside functions as in~\eqref{eqn:B}:
	\begin{align}
		\del_{\mathrm{RR}}^{\eps, \del}(t) %
		&= \underbrace{1 - e^t(1-\del)}_{I_1(t)}  + H(t-\eps) \cdot \underbrace{\frac{(e^t - e^\eps)(1-\del)}{e^\eps+1}}_{I_2(t)} + H(t+\eps)\cdot\underbrace{\frac{(e^{t+\eps}-1)(1-\del)}{e^{\eps}+1}}_{I_3(t)}.
	\end{align}
	Then, by chain rule, its first and second derivatives are:
	\begin{align}
		\dot{\del}^{\eps, \del}_\mathrm{RR}(t) &= \dot I_1(t) + \left(H(t-\eps) \cdot\dot I_2(t) + \underbrace{\dirac(t-\eps) \cdot I_2(t)}_{J_2(t)} \right) + \left(H(t+\eps) \cdot \dot I_3(t) + \underbrace{\dirac(t+\eps) \cdot I_3(t)}_{J_3(t)}\right) \\
		\ddot{\del}^{\eps, \del}_\mathrm{RR}(t) &= \ddot I_1(t) + \left(H(t-\eps) \cdot \ddot I_2(t) +  \dirac(t-\eps) \cdot \dot I_2(t) + \dot J_2(t)\right) \nonumber \\
							&\quad + \left(H(t+\eps) \cdot \ddot I_3(t) + \dirac(t+\eps) \cdot \dot I_3(t) + \dot J_3(t)\right)
	\end{align}
	Note that $\dot I_1(t) = \ddot I_1(t) = -e^t(1-\del)$, $\dot I_2(t) = \ddot I_2(t) = e^t \cdot \frac{(1-\del)}{e^\eps + 1}$ and $\dot I_3(t) = \ddot I_3(t) = e^t \cdot \frac{e^{\eps}(1-\del) }{e^\eps + 1}$. Therefore, on subtracting the two, a lot of terms cancel out, and we get:
	\begin{align}
		\ddot{\del}^{\eps, \del}_\mathrm{RR}(t) - \dot{\del}^{\eps, \del}_\mathrm{RR}(t) &= \dirac(t-\eps) \cdot (\dot I_2(t) - I_2(t)) +  \dot J_2(t) + \dirac(t+\eps)\cdot(\dot I_3(t) - I_3(t)) + \dot J_3(t) \\
		&= \dirac(t-\eps) \cdot \frac{(1-\del)e^\eps}{e^\eps + 1} + \dot J_2(t) + \dirac(t+\eps) \cdot \frac{(1-\del)}{e^\eps + 1}  + \dot J_3(t).
	\end{align}
	Note that $J_2(t) = 0$ everywhere except at $t = \eps$. Similarly, $J_3(t) = 0$ everywhere except $t= -\eps$.
	
	On substituting and convolving, we get
	\begin{align}
		\del_{P_{1:l}|Q_{1:l}}(t) &\leq \int_{-\infty}^\infty \del^{\otimes l-1}(t-\tau) \times \left(\ddot \del^{\eps_l, \del_l}_\mathrm{RR}(\tau) - \dot \del^{\eps_l, \del_l}_\mathrm{RR}(\tau)\right) \mathrm{d}\tau \\
		&= \int_{-\infty}^\infty \del^{\otimes l-1}(t-\tau) \times \left(\dirac(\tau-\eps_l) \cdot \frac{(1-\del_l)e^{\eps_l}}{e^{\eps_l} + 1} + \dirac(\tau+\eps_l) \cdot \frac{(1-\del_l)}{e^{\eps_l} + 1} + \dot J_2(\tau) + \dot J_3(\tau)\right)\mathrm{d}\tau \\
		&= \del^{\otimes l-1}(t-\eps_l) \cdot \frac{(1-\del_l)e^{\eps_l}}{e^{\eps_l}+1} + \del^{\otimes l-1}(t+\eps_l) \cdot \frac{1-\del_l}{e^{\eps_l}+1} \\
		&\quad+ \underbrace{\int_{-\infty}^\infty \del^{\otimes l-1}(t-\tau) \times\dot J_2(\tau)  \mathrm{d}\tau}_{K_2(t)} + \underbrace{\int_{-\infty}^\infty \del^{\otimes l-1}(t-\tau) \times\dot J_3(\tau)  \mathrm{d}\tau}_{K_3(t)}
	\end{align}
	For the last integral, apply integration by parts to get
	\begin{align}
		K_2(t) &= \int_{-\infty}^\infty \del^{\otimes l-1}(t-\tau) \dot I_2(\tau) \cdot \dirac(\tau - \eps_l)\mathrm{d} \tau + \int_{-\infty}^\infty \del^{\otimes l-1}(t-\tau)I_2(\tau) \cdot \dot\dirac(\tau - \eps_l) \mathrm{d}\tau \\
		&= \del^{\otimes l-1}(t-\eps_l) \dot I_2(\eps_l) - \left(\del^{\otimes l-1}(t-\eps_l) \dot I_2(\eps_l) + \dot \del^{\otimes l-1}(t-\eps_l) I_2(\eps_l) \right) \\
		&= 0 - 0
	\end{align}
	since $I_2(\eps_l) = 0$ and for any function $f$ on $\R$ it holds that, $\int_\R f \cdot \dot \dirac \mathrm{d}\tau = - \int_{\R} \dot f \cdot \dirac \mathrm{d}\tau$. Similarly,
	\begin{align}
		K_3(t) &= \int_{-\infty}^\infty \del^{\otimes l-1}(t-\tau) \dot I_3(\tau) \cdot \dirac(\tau + \eps_l)\mathrm{d} \tau + \int_{-\infty}^\infty \del^{\otimes l-1}(t-\tau)I_3(\tau) \cdot \dot\dirac(\tau + \eps_l) \mathrm{d}\tau \\
		&= \del^{\otimes l-1}(t+\eps_l) \dot I_3(-\eps_l) - \del^{\otimes l-1}(t+\eps_l) \dot I_3(-\eps_l) -  \dot \del^{\otimes l-1}(t-\eps_l) I_3(-\eps_l) \\
		&= 0 - 0
	\end{align}
	since $I_3(-\eps_l) = 0$. Therefore, we have
	\begin{equation}
		\del_{P_{1:l}|Q_{1:l}}(t)  \leq \frac{(1-\del_l)}{e^{\eps_l}+1}\left[e^{\eps_l} \cdot  \del^{\otimes l-1}(t-\eps_l) +  \del^{\otimes l-1}(t+\eps_l) \right] \leq \del^{\otimes l}(t).
	\end{equation}
	Hence, the induction statement holds.
\end{proof}

\begin{repcorollary}{corr:approxdp_composition_homo}
	For any $\eps \geq 0$, $\del \in [0,1]$, the $k$-fold composition of $(\eps, \del)$-DP mechanisms satisfies $(\eps, \del^{\otimes k}(\eps))$-DP for all $\eps$, where
	\begin{equation}
		\label{eqn:approxdp_composition_homp_app}
		\forall t \in \R \  : \ \del^{\otimes k}(t) = 1 - (1-\del)^k \left( 1- \expec{Y \leftarrow  \mathrm{Binomial}\left(k, \frac{e^\eps}{1+e^\eps}\right)}{1 - e^{t-\eps \cdot (2Y - k)}}_+\right).
	\end{equation}
\end{repcorollary}
\begin{proof}
	We just have to show that recurrence relationship in~\autoref{thm:approxdp_composition_hetro}, that is $\forall l \in \{1, \cdots, k\}$
	\begin{equation}
		\forall t \in \R \ : \ \del^{\otimes l}(t) = \del_l + \frac{(1-\del_l)}{e^{\eps_l} + 1}\left[ e^{\eps_l} \cdot \del^{\otimes l - 1}(t -\eps_l) + \del^{\otimes l-1}(t +\eps_l)\right], \quad \text{where} \quad \del^{\otimes 0}(t) = [1 - e^t]_+,
	\end{equation}
	simplifies to the theorem statement when $\eps_i = \eps_j = \eps$ and $\del_i = \del_j = \del$ for all $i, j \in \{1, \cdots, k\}$. Let's define $p = \frac{e^\eps}{e^\eps + 1}$. The recurrence can then be stated as
	\begin{align}
		\del^{\otimes l}(t) &= \del + (1-\del) \expec{Y_l \leftarrow \mathrm{Bernoulli}(p)}{\del^{\otimes l-1}(t-\eps(2Y_l - 1))} \\
				    &=\del + (1-\del) \expec{Y_l \leftarrow \mathrm{Bernoulli}(p)}{\del + (1-\del) \expec{Y_{l-1} \leftarrow \mathrm{Bernoulli}(p)}{\del^{\otimes l-2} (t - \eps(2 Y_l - 1) - \eps(2Y_{l-1}-1))}} \\
				    &= \del\sum_{i=1}^2 (1-\del)^{i-1} + (1-\del)^{2}\expec{\substack{Y_l \leftarrow \mathrm{Bernoulli}(p) \\ Y_{l-1} \leftarrow \mathrm{Bernoulli(p)}}}{\del^{\otimes l-2} (t - \sum_{i=l-1}^l\eps(2 \cdot Y_i  - 1))} \\
				    &= \del\sum_{i=1}^l (1-\del)^{i-1} + (1-\del)^{l}\expec{\substack{Y_l \leftarrow \mathrm{Bernoulli}(p) \\ Y_{l-1} \leftarrow \mathrm{Bernoulli(p)}}}{\del^{\otimes 0} (t - \sum_{i=1}^l\eps(2 \cdot Y_i  - 1))} \\
				    &= 1 - (1-\del)^l +(1-\del)^l \expec{Y \leftarrow  \mathrm{Binomial}\left(k, p\right)}{\del^{\otimes 0} (t - \eps(2\cdot Y - l))} \\
				    &= 1 -(1-\del)^l\left(1 - \expec{Y \leftarrow  \mathrm{Binomial}\left(k, p\right)}{1 - e^{t-\eps \cdot (2Y - k)}}_+\right).
	\end{align}
\end{proof}

\subsection{Deferred Proofs for Section~\ref{sec:subsampling}}
\label{app:subsampling}

\begin{reptheorem}{thm:subsampling}[Poisson subsampling]
	Let $0 \leq \lambda \leq 1$. For any two distributions $P$ and $Q$ on $\OO$, 
	\begin{align}
		\del_{\lambda P + (1-\lambda) | Q}(\eps) = \begin{cases} \lambda \del_{P|Q}(\log(1 + (e^{\eps} - 1)/\lambda)) & \text{if } \eps > \log(1-\lambda) \\1 - e^\eps &\text{otherwise} \end{cases}.
	\end{align}
\end{reptheorem}

\begin{proof}
	Recall from Definition~\ref{def:pld} that $\privloss{Q}{P}$ and $\privloss{Q}{\lambda P + (1-\lambda)Q}$ are the distributions of $\priv{Q}{P}(\Thet)$ and $\priv{Q}{P\lambda + (1-\lambda)Q}(\Thet)$ respectively with $\Thet \sim Q$. Since for any $\thet \in \OO$, 
	\begin{equation}
		\priv{Q}{\lambda P + (1-\lambda)Q}(\thet) = - \log \frac{\lambda P(\thet) + (1-\lambda) Q(\thet)}{Q(\thet)} = - \log(1 - \lambda + \lambda \cdot e^{-\priv{Q}{P}(\thet)}),
	\end{equation}
	the random variables $\Z'_\lambda \sim \privloss{Q}{\lambda P + (1-\lambda) Q}$ and $\Z' \sim \privloss{Q}{P}$ are related as
	\begin{equation}
		\Z'_\lambda = -\log(1 + \lambda(e^{-\Z'} -1)).
	\end{equation}
	Using Theorem~\ref{thm:dp_as_laplace} that we can express the privacy profile as:
	\begin{align}
		\del_{\lambda P + (1-\lambda) Q|Q}(\eps) &= e^\eps\UniLap{F_{\Z'_\lambda}(-t - \eps)}(-1) \\
							 &\overset{\eqref{eq:Differentiation}}{=} e^\eps\left[\UniLap{f_{\Z'_\lambda}(-t-\eps)}(-1) - F_{\Z'_\lambda}(\eps^-)\right] \\
							 &= e^\eps \left[\int_{0^+}^\infty e^t f_{\Z'_\lambda}(-t-\eps)dt - \int_{-\infty}^{0^-}f_{\Z'_\lambda}(t-\eps)dt \right] \\
							 &=e^\eps \int^{0^-}_{-\infty} (e^{-t} - 1)f_{\Z'_\lambda}(t-\eps)dt \\
							 &= \int_{-\infty}^{-\eps^-} \left(e^{-t'} - e^\eps\right) f_{\Z'_\lambda}(t')dt' \\
							 &= \expec{}{e^{-\Z'_\lambda} - e^\eps}_+,
	\end{align}
	where $[x]_+ \eqdef \max\{0, x\}$.
	On substituting $\Z'_\lambda$, we get:
	\begin{align}
		\del_{\lambda P + (1-\lambda)Q| Q}(\eps) &= \expec{}{e^{-\Z'_\lambda} - e^\eps}_+ \\
							 &=\expec{}{1-\lambda + \lambda e^{-\Z'} - e^\eps}_+ \\
							 &= \lambda\expec{}{e^{-\Z'} - \frac{e^\eps +\lambda - 1}{\lambda}}_+ \\
							 &= \begin{cases}\lambda \expec{}{e^{-\Z'} - e^{\log(1 + (e^{\eps} - 1)/\lambda)}}_+ &\text{if } \eps > \log (1-\lambda)\\ \lambda \expec{}{e^{-\Z'}} + 1-\lambda - e^\eps &\text{otherwise} \end{cases}\\
							 &= \begin{cases} \lambda \del_{P|Q}(\log(1 + (e^{\eps} - 1)/\lambda)) & \text{if } \eps > \log(1-\lambda) \\1 - e^\eps &\text{otherwise} \end{cases}. \tag{$\because \expec{}{e^{-\Z'}} = 1$}
	\end{align}

\end{proof}

\newpage

\begin{figure}[!ht]
	\centering
	\includegraphics[width=\linewidth]{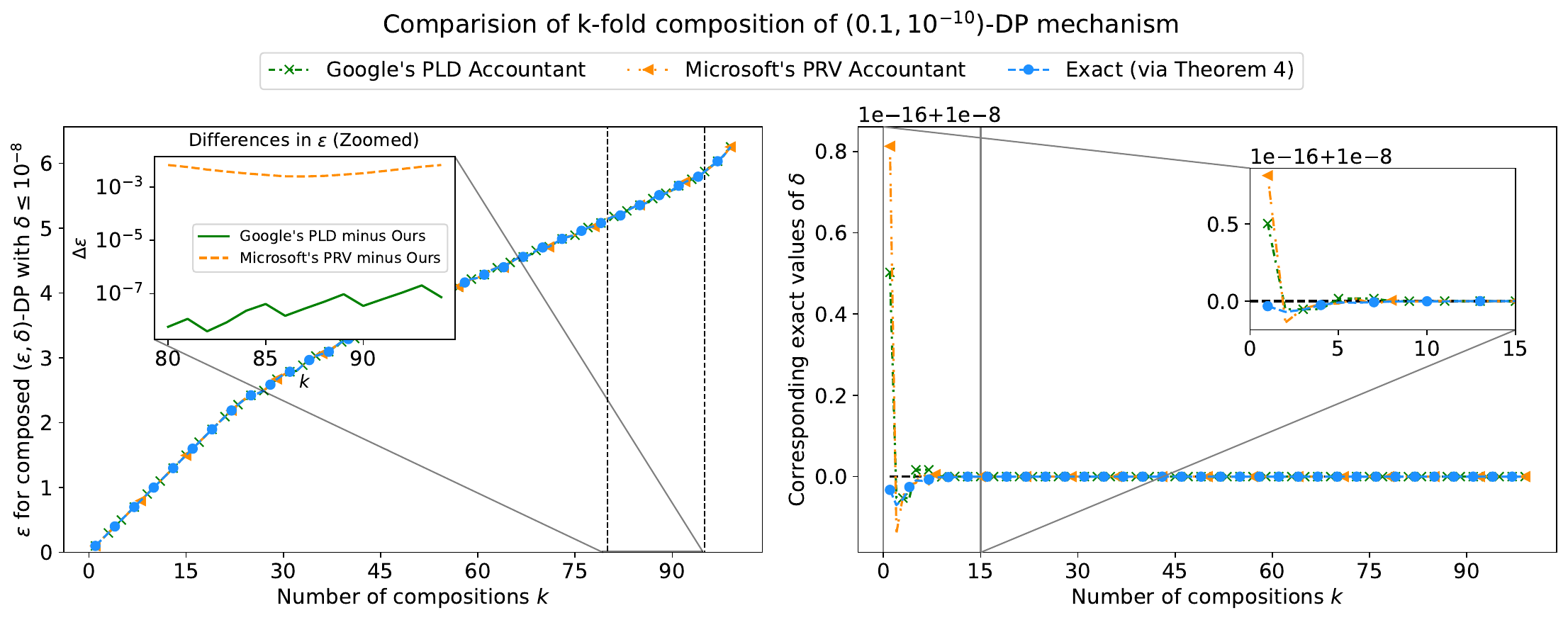}
	\caption[caption of figure]{\label{fig:composition_numerical_compare}
		Comparison of $(\eps, \del)$-DP bounds between numerical accountants and our~\autoref{corr:approxdp_composition_homo} for $100$-fold composition of a $(0.1, 10^{-10})$-DP point guarantee, with the budget constraint $\del < 10^{-8}$. Note that at $k=100$, the constraint on $\del$ cannot be satisfied and so the corresponding $\eps = \infty$ at that value. We note that at smaller values of $k$, numerical accountants can sometimes over exceed the budget constraints on $\del$. Additionally, the gap for $\eps$ between our exact bound and those approximated by numerical accountant tend to be of the order $\approx 10^{-7}$ for Google's PLDAccountant and $\approx 10^{-3}$ for Microsoft's PRVAccountant.
	}
\end{figure}

\end{document}